%% file: arxivmain.tex
\documentclass{article}
\usepackage[margin=1in]{geometry}

\usepackage{hyperref}

\usepackage[round]{natbib}
\usepackage[utf8]{inputenc} 
\usepackage[T1]{fontenc}    
\usepackage{hyperref}       
\usepackage{url}            
\usepackage{booktabs}       
\usepackage{amsfonts}       
\usepackage{nicefrac}       
\usepackage{microtype}      
\usepackage{xcolor}         

\usepackage{amsthm}
\usepackage{amssymb}
\usepackage{amsmath}
\usepackage{graphicx}
\usepackage{comment} 

\usepackage{hyperref}  
\definecolor{mydarkblue}{rgb}{0,0.08,0.45}
\hypersetup{colorlinks=true, citecolor=mydarkblue,linkcolor=mydarkblue}
\usepackage{refcount}

\newtheorem{theorem}{Theorem}[section]
\newtheorem{claim}[theorem]{Claim}
\newtheorem{definition}[theorem]{Definition}
\newtheorem{lemma}[theorem]{Lemma}
\newtheorem{fact}[theorem]{Fact}

\newcommand{\range}{{\rm range}}
\newcommand{\argmin}{{\rm argmin}}
\newcommand{\Cone}{{\rm Cone}}
\newcommand{\dom}{{\rm dom}}
\newcommand{\argmax}{{\rm argmax}}
\newcommand{\erf}{{\rm erf}}

\newtheorem{mydef}{Definition}[section]
\newtheorem{lem}[mydef]{Lemma}

\newtheorem{cor}[mydef]{Corollary}

\newtheorem{ass}{Assumption}[section]

\newcommand{\wt}{\widetilde}
\newcommand{\ov}{\overline}

\newcommand{\N}{\mathcal{N}}

\newcommand{\R}{\mathbb{R}}

\renewcommand{\i}{\mathbf{i}} 
\renewcommand{\varepsilon}{\epsilon}
\renewcommand{\tilde}{\wt}

\renewcommand{\bar}{\overline}

\renewcommand{\d}{\mathsf{d}}
\newcommand{\poly}{\mathrm{poly}}

\newcommand{\sgn}{\mathrm{sgn}}

\DeclareMathOperator*{\E}{{\mathbb{E}}}
\DeclareMathOperator*{\Var}{{\bf {Var}}}

\DeclareMathOperator{\dis}{dis}
\DeclareMathOperator{\cts}{cts}

\title{\vspace{-2em}A New Initialization Technique for Reducing the Width of Neural Networks}

\author{Alexander Munteanu\thanks{Dortmund Data Science Center, Faculties of Statistics and Computer Science, TU Dortmund University, Dortmund, Germany. Email: \texttt{alexander.munteanu@tu-dortmund.de}.}
\and Simon Omlor \thanks{Faculty of Statistics, TU Dortmund University, Dortmund, Germany. Email: \texttt{simon.omlor@tu-dortmund.de}.}
\and Zhao Song \thanks{Adobe Research. Email: \texttt{zsong@adobe.com}.}
\and David P. Woodruff \thanks{Department of Computer Science, Carnegie Mellon University, Pittsburgh, PA 15213, USA. Email: \texttt{dwoodruf@cs.cmu.edu}.}
}

\title{Bounding the Width of Neural Networks via Coupled Initialization - A Worst Case Analysis}

\begin{document}

\allowdisplaybreaks
\maketitle

\begin{abstract}
  \input{arxivabstract}
\end{abstract}

\input{arxivintro}
\input{arxivproblem}

\input{arxivresults}
\input{arxivtech}

\input{arxivdiscuss}

\section*{Acknowledgements}
We thank the anonymous reviewers and Binghui Peng for their valuable comments. Alexander Munteanu and Simon Omlor were supported by the German Research Foundation (DFG), Collaborative Research Center SFB 876, project C4 and by the Dortmund Data Science Center (DoDSc). David Woodruff would like to thank NSF grant No. CCF-1815840, NIH grant 5401 HG 10798-2, ONR grant N00014-18-1-2562, and a Simons Investigator Award. 

\bibliographystyle{plainnat}
\bibliography{ref}

\appendix
\onecolumn
\section*{Appendix}

\input{arxivprobability}
\input{arxivlogwidth}
\input{arxivanalysis}

\input{arxivtechnical}

\end{document}

%% file: arxivabstract.tex
A common method in training neural networks is to initialize all the weights to be independent Gaussian vectors. We observe that by instead initializing the weights into independent pairs, where each pair consists of two identical Gaussian vectors, we can significantly improve the convergence analysis. While a similar technique has been studied for random inputs [Daniely, NeurIPS 2020], it has not been analyzed with arbitrary inputs. Using this technique, we show how to significantly reduce the number of neurons required for two-layer ReLU networks, both in the under-parameterized setting with logistic loss, from roughly $\gamma^{-8}$ [Ji and Telgarsky, ICLR 2020] to $\gamma^{-2}$, where $\gamma$ denotes the separation margin with a Neural Tangent Kernel, as well as in the over-parameterized setting with squared loss, from roughly $n^4$ [Song and Yang, 2019] to $n^2$, implicitly also improving the recent running time bound of [Brand, Peng, Song and Weinstein, ITCS 2021]. For the under-parameterized setting we also prove new lower bounds that improve upon prior work, and that under certain assumptions, are best possible. 

%% file: arxivintro.tex
\section{Introduction}
Deep learning has achieved state-of-the-art performance in many areas, e.g., computer vision \cite{lbbh98,ksh12,slj+15,hzrs16}, natural language processing \cite{cwb+11,dclt18}, self-driving cars, games \cite{alphago16,alphago17}, and so on. A beautiful work connected the convergence of training algorithms for over-parameterized neural networks to kernel ridge regression, where the kernel is the Neural Tangent Kernel (NTK) \cite{jgh18}. 

The convergence results motivated by NTK mainly require two assumptions: (1) the kernel matrix $K$ formed by the input data points has a sufficiently large minimum eigenvalue $\lambda_{\mathrm{min}}(K)\geq\lambda>0$, which is implied by the separability of the input point set \cite{os20}, and (2) the neural network is over-parameterized. Mathematically, the latter means that the width of the neural network is a sufficiently large polynomial in the other parameters of the network, such as the number of input points, the data dimension, etc.
The major weakness of such convergence results is that the neural network has to be sufficiently over-parameterized. In other words, the over-parameterization is a rather large polynomial, which is not consistent with architectures for neural networks used in practice, cf. \cite{kh19}. 

Suppose $m$ is the width of the neural network, which is the number of neurons in a hidden layer, and $n$ is the number of input data points. In an attempt to reduce the number of neurons for binary classification, a recent work \cite{jt20} has shown that a polylogarithmic dependence on $n$ suffices to achieve arbitrarily small training error. Their width, however, depends on the separation margin $\gamma$ in the RKHS (Reproducing Kernel Hilbert Space) induced by the NTK. More specifically they show an upper bound of $m=O({\gamma^{-8}}{\log n })$ and a lower bound of $m=\Omega({\gamma}^{-1/2})$ relying on the NTK technique. Our new analysis in this regime significantly improves the upper bound to $m=O(\gamma^{-2} \log n )$. 

We complement this result with a series of lower bounds. Without relying on any assumptions we show $m=\Omega(\gamma^{-1})$ is necessary. Assuming we need to rely on the NTK technique as in \cite{jt20}, we can improve their lower bound 
to $m=\Omega(\gamma^{-1}{\log n})$. Finally, assuming we need to rely on a special but natural choice for estimating an expectation by its empirical mean in the analysis of \citet{jt20}, which we have adopted in our general upper bound, we can even prove that $m=\Theta(\gamma^{-2}{\log n })$, i.e., that our analysis is tight. However, in the $2$-dimensional case we can construct a better estimator yielding a linear upper bound of $m=O({\gamma^{-1}}{\log n})$, so the above assumption seems strong for very low dimensions, though it is a seemingly natural method that works in arbitrary dimensions. We also present a candidate hard instance in $\Theta(\log \gamma^{-1})$ dimensions which could potentially give a matching $\Omega(\gamma^{-2})$ lower bound, up to logarithmic factors. 

For regression with target variable $y$ with $|y| \in O(1)$ we consider a two-layer neural network with squared loss and ReLU as the  activation function, which is standard and popular in the study of deep learning. \citet{dzps19} show that $m = O( \lambda^{-4} n^6 )$ suffices (suppressing the dependence on remaining parameters). Further, \citet{sy19} improve this bound to $m = O(\lambda^{-4} n^4)$. The trivial information-theoretic lower bound is $\Omega(n)$, since the model has to memorize\footnote{Here, by memorize, we mean that the network has zero error on every input point.} the $n$ input data points arbitrarily well. There remains a huge gap between $n$ and $n^4$. In this work, we improve the upper bound, showing that $m = O(\lambda^{-2} n^2)$ suffices for gradient descent to get arbitrarily close to $0$ training error. We summarize our results and compare with previous work in Table~\ref{tab:results}.

\subsection{Related Work}

The theory of neural networks is a huge and quickly growing field. Here we only give a brief summary of the work most closely related to ours.

\textbf{Convergence results for neural networks with random inputs.}
Assuming the input data points are sampled from a Gaussian distribution is often done for proving convergence results \cite{zsjbd17,ly17,zsd17,glm18,bjm19,ckm20}. A more closely related work is the work of \citet{Daniely20} who introduced the coupled initialization technique, and showed that $\tilde O(n/d)$ hidden neurons can memorize all but an $\epsilon$ fraction of $n$ random binary labels of points uniformly distributed on the sphere. Similar results were obtained for random vertices of a unit hypercube and for random orthonormal basis vectors. In contrast to our work, this reference uses \emph{stochastic} gradient descent, where the nice assumption on the input distribution gives rise to the $1/d$ factor; however, this reference achieves only an approximate memorization. We note that full memorization of \emph{all} input points is needed to achieve our goal of an error arbitrarily close to zero, and $\Omega(n)$ neurons are needed for worst case inputs. Similarly, though not necessarily relying on random inputs, \citet{belm20} shows that for \emph{well-dispersed} inputs, the neural tangent kernel (with ReLU network) can memorize the input data with
$\tilde O(n/d)$ neurons. However, their training algorithm is neither a gradient descent nor a stochastic gradient descent algorithm, and also their network consists of complex weights rather than real weights. One motivation of our work is to analyze standard algorithms such as gradient descent.
In this work, we do not make any input distribution assumptions; therefore, these works are incomparable to ours. In particular,  random data sets are often well-dispersed inputs
that allow smaller width and tighter concentration, but are hardly realistic. In contrast, we conduct worst case analyses to cover all possible inputs, which might not be well-dispersed in practice.

\textbf{Convergence results of neural networks in the under-parameterized setting.}
When considering classification with cross-entropy (logistic) loss, the analogue of the minimum eigenvalue parameter of the kernel matrix is the maximum separation margin $\gamma$ (see Assumption ~\ref{ass:margin_gamma} for a formal definition) in the RKHS of the NTK.
Previous separability assumptions on an infinite-width two-layer ReLU network in \cite{cg19,CaoGu19} and on smooth target functions in \cite{all19} led to polynomial dependencies between the width $m$ and the number $n$ of input points. The work of \cite{nitanda2019gradient} relies on the NTK separation mentioned above and improved the dependence, but was still polynomial.

A recent work of \cite{jt20} gives the first convergence result based on an NTK analysis where the \emph{direct} dependence on $n$, i.e., the number of points, is only poly-logarithmic. Specifically, they show that as long as the width of the neural network is polynomially larger than $1/\gamma$ and $\log n$, then gradient descent can achieve zero training loss.

\textbf{Convergence results for neural networks in the over-parameterized setting.}
There is a body of work studying convergence results of over-parameterized neural networks \cite{ll18,dzps19,als19_rnn,als19_dnn,dllwz19,all19,sy19,adhlw19,adhlsw19,cg19,zg19,dhp+19,lsswy20,hy20,cx20,bpsw21,ljzkd21,syz21}. One line of work explicitly works on the neural tangent kernel \cite{jgh18} with kernel matrix $K$.
This line of work shows that as long as the width of the neural network is polynomially larger than $n/\lambda_{\min}(K)$, then one can achieve zero training error. Another line of work instead assumes that the  input data points are not too ``collinear'', where this is formalized by the parameter $\delta = \min_{i\neq j} \{ \| x_i - x_j \|_2 , \| x_i + x_j \|_2 \}$\footnote{This is also sometimes called the separability of data points.} \cite{ll18,os20}. 
These works show that as long as the width of the neural network is polynomially larger than $1/\delta$ and $n$, then one can train the neural network to achieve zero training error. 
The work of \citet{sy19} shows that the over-parameterization $m=\Omega(\lambda^{-4}n^4)$ suffices for the same regime we consider\footnote{Although the title of \cite{sy19} is quadratic, $n^2$ is only achieved when the finite sample kernel matrix deviates from its limit in norm only by a constant $\alpha$ w.h.p., and the inputs are \emph{well-dispersed} with constant $\theta$, i.e., $|\langle x_i, x_j\rangle|\leq \theta/\sqrt{n}$ for all $i\neq j$. In general, \cite{sy19} only achieve a bound of $n^4$.
}. Additional work claims that even a linear dependence is possible, though it is in a different setting. E.g., \cite{kh19} show that for any neural network with nearly linear width, there exists a trainable data set. Although their width is small, this work does not provide a general convergence result.
Similarly, \citet{ZhangZHSL21} use a coupled LeCun initialization scheme that also forces the output at initalization to be $0$. This is shown to improve the width bounds for shallow networks below $n$ neurons. However, their convergence analysis is local and restricted to cases where it remains unclear how to find {globally} optimal or even approximate solutions. We instead focus on cases where gradient descent provably optimizes up to arbitrary small error, for which we give a lower bound of $\Omega(n)$. 

Other than considering over-parameterization in first-order optimization algorithms, such as gradient descent, \citet{bpsw21} show convergence results via second-order optimization, such as Newton's method. Their running time also relies on $m = \Omega(\lambda^{-4}n^4)$, which is the state-of-the-art width for first-order methods \cite{sy19}, and it was noted that any improvement to $m$ would yield an improved running time bound.

Our work presented in this paper continues and improves those lines of research on understanding two-layer ReLU networks.

\paragraph{Roadmap.}
In Section~\ref{sec:problem}, we introduce our problem formulations and present our main ideas. In Section~\ref{sec:ourresults}, we present our main results. In Section~\ref{sec:techoverview}, we present a technical overview of our core analysis for the convergence of the gradient descent algorithm in both of our studied regimes and give a hard instance and the intuition behind our lower bounds. In Section~\ref{sec:discuss}, we conclude our paper with a summary and some discussion.

We defer all detailed technical proofs to the appendix. The details for the logarithmic width networks under logistic loss are given in Appendices \ref{sec:logwidth:start}-\ref{sec:logwidth:end}, whereas the polynomial width networks with squared loss are analyzed in Appendices \ref{sec:analysis_concentration}-\ref{sec:missing_proof}.

%% file: arxivproblem.tex
\section{Problem Formulation and Initialization Scheme}\label{sec:problem}

We follow the standard problem formulation \cite{dzps19,sy19,jt20}. One major difference of our formulation with the previous work is that we do not have a $1/\sqrt{m}$ normalization factor in what follows.
We note that only removing the normalization does not give any improvement in the amount of over-parameterization required of the previous bounds. 
The output function of our network is given by \begin{equation}\label{eqn:main}
f (W,x,a) =  \sum_{r=1}^m a_r \phi ( w_r^\top x ) ,
\end{equation} where $\phi(z)=\max\{z,0\}$ denotes the ReLU activation function\footnote{We note that our analysis can be extended to Lipschitz continuous, positively homogeneous activations.}, $x \in \R^d$ is an input point, $w_1, \ldots, w_m \in \R^d$ are weight vectors in the first (hidden) layer, and $a_1, \ldots, a_m \in \{-1, +1\}$ are weights in the second layer. We only optimize $W$ and keep $a$ fixed, which suffices to achieve zero error. Also previous work shows how to extend the analysis to include $a$ in the optimization, cf. \cite{dzps19}.
\begin{definition}[Coupled Initialization]\label{def:initialization}
We initialize the network weights as follows:
\begin{itemize}
	\item For each $r = 2i-1$, we choose $w_r$ to be a random Gaussian vector drawn from $\mathcal{N}(0,I)$. 
	\item For each $r = 2i-1$, we sample $a_r$ from $\{-1,+1\}$ uniformly at random. 
	\item For each $r = 2i$, we choose $w_{r} = w_{r-1}$. 
	\item For each $r = 2i$, we choose $a_{r} = -a_{r-1}$.
\end{itemize}

\end{definition}
We note this coupled initialization appeared before in \cite{Daniely20} for analyzing well-spread random inputs on the sphere. The initialization is chosen in such a way as to ensure that for each of the $n$ input points, the initial value of the network is $0$. Here we present an independent and novel analysis, where this property is leveraged \emph{repeatedly} to bound the iterations of the optimization, which yields significantly improved worst case bounds for any input. This is crucial for our analysis, and is precisely what allows us to remove the $1/\sqrt{m}$ factor that multiplies the right-hand-side of (\ref{eqn:main}) in previous work. Indeed, that factor was there precisely to ensure that the initial value of the network is small. One might worry that our initialization causes the weights to be {\it dependent}. Indeed, each weight vector occurs exactly twice in the hidden layer. We are able to show that this dependence does not cause problems for our analysis. In particular, the minimum eigenvalue bounds of the associated kernel matrix and the separation margin in the NTK-induced feature space required for convergence in previous work can be shown to still hold, since such analyses are loose enough to accommodate such dependencies. Now, we have a similar initialization as in previous work, but since we no longer need a $1/\sqrt{m}$ factor in (\ref{eqn:main}), we can show that we can change the learning rate of gradient descent from that in previous work and it no longer needs to be balanced with the initial value, since the latter is $0$. This ultimately allows for us to use a smaller over-parameterization (i.e., value of $m$) in our analyses.
For $r\in [m]$,
we have\footnote{Note that ReLU is not continuously differentiable. Slightly abusing notation, one can view $\partial f/\partial w_r$ as a valid (sub)gradient given in the RHS of (2). This extends to $\partial L/\partial w_r$ as the RHS of (4) and (5) which yields the update rule (6) commonly used in practice and in related theoretical work, cf. \cite{dzps19}.
}
\begin{align}\label{eq:relu_derivative}
\frac{\partial f(W,x,a)}{\partial w_r}=a_r x{\bf 1}_{ w_r^\top x \geq 0 }
\end{align}
independent of the loss function that we aim to minimize.

\subsection{Loss Functions}
In this work, we mainly focus on two different types of loss functions. The binary \emph{cross-entropy (logistic) loss} and the \emph{squared loss}. These loss functions are arguably the most well-studied for binary classification and for regression tasks with low numerical error, respectively.

We are given a set of $n$ input data points and corresponding labels, denoted by
\begin{align*}
\{ (x_1, y_1), \ldots, (x_n,y_n) \} \subset \R^d \times \R.
\end{align*}
We make a standard normalization assumption, as in \cite{dzps19,sy19,jt20}. In the case of logistic loss, the labels are restricted to $y_i \in \{-1,+1\}$. In the case of squared loss, the labels are $|y_i|=O(1)$. In both cases, as in prior work and for simplicity, we assume that $\| x_i \|_2 = 1$\footnote{
We adopt the assumption for a concise presentation, but we note it can be resolved by weaker constant bounds $0< \mathrm{lb} \leq \|x_i\| \leq \mathrm{ub}$, introducing a constant $\mathrm{ub}/\mathrm{lb}$ factor, cf. \cite{dzps19}, or otherwise the data can be rescaled and padded with an additional coordinate to ensure $\|x_i\|=1$, cf. \cite{all19}.
}, $\forall i \in [n]$. We also define the output function on input $x_i$ to be $f_i(W)=f (W,x_i,a)$. At time $t$,
let $u( W(t) )=(u_1( W(t) ),\ldots,u_n(W(t)))\in \mathbb{R}^n$ be the prediction vector, where each $u_i(W(t))$ is defined to be
\begin{align}\label{eq:ut_def}
u_i(W(t))=f(W(t),x_i,a).
\end{align}
For simplicity, we use $u(t)$ to denote $u(W(t))$ in later discussion.

We consider the objective function $L$: 
\begin{align*}
    L (W) = \sum_{i=1}^n \ell(y_i, u_i(W) )
\end{align*}
where the individual logistic loss is defined as $\ell(v_1,v_2)=\ln(1+\exp(-v_1 v_2))$, and the individual squared loss is given by $\ell(v_1,v_2)=\frac{1}{2}(v_1-v_2)^2$.

\addtocounter{footnote}{-2}
For logistic loss, we can compute the gradient\footnotemark~of $L$ in terms of $w_r \in \R^d$
\begin{align}\label{eq:gradient2}
\frac{\partial L(W)}{\partial w_r}=\sum_{i=1}^n \frac{-\exp(-y_if(W,x_i,a))}{1+\exp(-y_if(W,x_i,a))} y_i a_r x_i {\bf 1}_{ w_r^\top x_i \geq 0 }
\end{align}

\addtocounter{footnote}{-1}
For squared loss, we can compute the gradient\footnotemark~
of $L$ in terms of $w_r \in \R^d$
\begin{align}\label{eq:gradient1}
\frac{ \partial L(W) }{ \partial w_r } = \sum_{i=1}^n ( f(W,x_i,a) - y_i ) a_r x_i {\bf 1}_{ w_r^\top x_i \geq 0 }.
\end{align}

We apply gradient descent to optimize the weight matrix $W$ with the following standard update rule,
\begin{align}\label{eq:w_update}
W(t+1) = W(t) - \eta \frac{ \partial L( W(t) ) }{ \partial W(t) } ,
\end{align}

where $0 <\eta \leq 1$ determines the step size.

In our analysis, we assume that $W$ consists of $m_0$ blocks of Gaussian vectors, where in each block, there are $B$ identical copies of the same Gaussian vector. Thus, we have $m = m_0 \cdot B$. Ultimately we show it already suffices to set $m_0 =m/2$ and $ B= 2$. We use $w_{r, b}$ to denote the $b$-th row of the $r$-th block, where $b \in [B]$ and $r \in [m_0]$. When there is no confusion, we also use $w_{r}$ to denote the $r$-th row of $W$, $r \in [m]$.

%% file: arxivresults.tex
\section{Our Results}\label{sec:ourresults}
Our results are summarized and compared to previous work in Table \ref{tab:results}.
Our first main result is an improved general upper bound for the width of a neural network for binary classification, where training is performed by minimizing the cross-entropy (logistic) loss. We need the following assumption which is standard in previous work in this regime \cite{jt20}.
\begin{ass}[informal version of Definition \ref{def1} and Assumption \ref{assu}]\label{ass:margin_gamma}
We assume that there exists a mapping $\bar v$ with $\|\bar v(z)\|_2\leq 1$ for all $z\in\mathbb{R}^d$ and margin $\gamma > 0$ such that
\begin{align*}
    \min\limits_{i\in[n]} \E_{w \sim {\cal N}(0,I_d) } [ y_i \langle \bar v(w),x_i \rangle \mathbf{1}[\langle w,x_i \rangle > 0] ] > \gamma~.
\end{align*}
\end{ass}

\begin{table*}[t!]
	\centering
	\begin{tabular}{ | l| l| l| l| l| }
		\hline
		{\bf References} & {\bf Width} $m$ & {\bf Iterations} $T$ & {\bf Loss function} \\ \hline \hline
		\cite{jt20} & $O(\gamma^{-8} \log n)$ & $O( {\epsilon^{-1}\gamma^{-2}}(\sqrt{\log n}+\log(1/\epsilon))^2 )$ & logistic loss \\ \hline
		Our work & $O( \gamma^{-2} \log n)$ & $O({\epsilon^{-1}\gamma^{-2}}{\log^2(1/\epsilon)})$ & logistic loss \\ \hline
		\cite{jt20} & $\Omega( \gamma^{-1/2})$ & N/A & logistic loss \\ \hline 
		Our work & $\Omega( \gamma^{-1} \log n)$ & N/A & logistic loss \\ 
		 \hline \hline 
		\cite{dzps19} & $O(\lambda^{-4} n^6)$ & $O(\lambda^{-2} n^2 \log(1/\epsilon))$ & squared loss \\ \hline 
		\cite{sy19} & $O(\lambda^{-4} n^4)$ & $O(\lambda^{-2} n^2 \log(1/\epsilon) )$ & squared loss \\ \hline 
		Our work & $O(\lambda^{-2} n^2 )$ & $O( \lambda^{-2} n^2 \log(1/\epsilon) )$ & squared loss \\ \hline
	\end{tabular}
	\caption{Summary of our results and comparison to previous work. The improvements are mainly in the dependence on the parameters $\lambda, \gamma, n$ affecting the width $m$. None of the results depend on the dimension $d$, except the lower bounds, which require $d \geq 2$. In both regimes the dependence on $\epsilon$ is the same as in previous literature. We note that the difference between regimes comes from different properties of the loss functions that affect the convergence rate, cf. \cite{nitanda2019gradient}. We want to remark that our squared loss result also implicitly improves the dependence on $m$ in the running time bound of \citet{bpsw21} (see Theorem 1.1, Remark 1.2, and Table 1 in \cite{bpsw21}).
	} \label{tab:results}
\end{table*}

Our theorem improves the previous best upper bound of \citet{jt20} from $O(\gamma^{-8}{\log n})$ to only $O(\gamma^{-2}{\log n})$. As a side effect, we also remove the dependence of the number $n$ of iterations.

\begin{theorem}[informal version of Theorem~\ref{mainthm}]\label{thm:logwidth:main_informal}
Given $n$ labeled data points in $d$-dimensional space, consider a two-layer ReLU neural network with width $m = \Omega( \gamma^{-2}{\log n} )$. Starting from a coupled initialization (Def. \ref{def:initialization}), for any accuracy $\epsilon \in (0,1)$, we can ensure the cross-entropy (logistic) training loss is less than $\epsilon$ when running gradient descent for $T= O( {\epsilon^{-1} \gamma^{-2}}{\log^2(1/\epsilon) })$ iterations.
\end{theorem}

As a corollary of Theorem~\ref{thm:logwidth:main_informal}, we immediately obtain the same significant improvement from  $O(\gamma^{-8}{\log n})$ to only $O(\gamma^{-2}{\log n})$ for the generalization results of \citet{jt20}. To this end, we first extend Assumption \ref{ass:margin_gamma} to hold for any data generating distribution instead of a fixed input data set:

\begin{ass}\label{ass:margin_gamma_dist}\cite{jt20}
    We assume that there exists a mapping $\bar v$ with $\|\bar v(z)\|_2\leq 1$ for all $z\in\mathbb{R}^d$ and margin $\gamma > 0$ such that
    \begin{align*}
        \E_{w \sim {\cal N}(0,I_d) } [ y \langle \bar v(w),x \rangle \mathbf{1}[\langle w,x \rangle > 0] ] > \gamma~
    \end{align*}
    for almost all $(x,y)$ sampled from the data distribution $\mathcal D$.
\end{ass}

By simply replacing the main result, Theorem 2.2 of \citet{jt20} by our Theorem \ref{thm:logwidth:main_informal} in their proof\footnote{We note that in Theorem \ref{thm:logwidth:main_informal} we did not bound the distance between the weights at each step $t\leq T$ compared to the initialization $t=0$. Since this can be done \emph{exactly} as in Theorem 2.2 of \citet{jt20}, we omit this detail for brevity of presentation.}, we obtain the following improved generalization bounds with full gradient descent:

\begin{cor}\label{cor:gen1}
Given a distribution $\mathcal D$ over labeled data points in $d$-dimensional space, consider a two-layer ReLU neural network with width $m = \Omega( \gamma^{-2}{\log n} )$. Starting from a coupled initialization (Def. \ref{def:initialization}), with constant probability over the data samples from $\mathcal D$ and over the random initialization, it holds for an absolute constant $C>1$ that
\[ \Pr_{(x,y)\sim\mathcal D} \left[  y f(W_k, x, a) \leq 0 \right] 
\leq 2\varepsilon + \frac{8 \ln (4/\varepsilon) }{\gamma^2 \sqrt{n}} + 6\sqrt{\frac{\ln(2C)}{2n}}, \]
where $k$ denotes the step attaining the smallest empirical risk before $T= O( {\epsilon^{-1} \gamma^{-2}}{\log^2(1/\epsilon) })$ iterations.
\end{cor}

Corollary \ref{cor:gen1} can then be used exactly as in \citep{jt20} to obtain:

\begin{cor}
Under Assumption \ref{ass:margin_gamma_dist}, given $ \varepsilon>0$, and a uniform random sample of size $n= \tilde{\Omega} ({\gamma^{-4} \varepsilon^{-2}})$ and $m=\Omega( \gamma^{-2}{\log n} ) $ it holds with constant probability that $\Pr_{(x,y)\sim\mathcal D} \left[  y f(W_k, x, a) \leq 0 \right] \leq \varepsilon$ where $k$ denotes the step attaining the smallest empirical risk before $T= O( {\epsilon^{-1} \gamma^{-2}}{\log^2(1/\epsilon) })$ iterations.
\end{cor}

We finally note that the improved generalization bound can be further extended exactly as in \cite{jt20} to work for \emph{stochastic} gradient descent.

Next, we turn our attention to lower bounds. We provide an unconditional linear lower bound, and note that Lemma \ref{Prop5.4'} yields an $m=\Omega(n)$ lower bound for any loss function, in particular also for squared loss; see Sec. \ref{sec:lb_informal}.
\begin{theorem}[informal version of Lemma \ref{Prop5.4'}]\label{LB:uncon_informal}
There exists a data set in $2$-dimensional space, such that any two-layer ReLU neural network with width $m = o({\gamma}^{-1})$ necessarily misclassifies at least $\Omega(n)$ points.
\end{theorem}

Next, we impose the same assumption as in \cite{jt20}, namely, that separability is possible at initialization of the NTK analysis. Formally, this means that there exists $V\in\mathbb{R}^{m\times d}$ such that no $i\in[n]$ has $y_i\langle \nabla f_i(W_0), V \rangle \leq 0$. Under this condition we improve their lower bound of $m=\Omega({\gamma^{-1/2}})$ to $m=\Omega(\gamma^{-1}{\log n})$:
\begin{theorem}[informal version of Lemma \ref{Prop5.4}]\label{LB:NTK_informal}
There exists a data set of size $n$ in $2$-dimensional space, such that for any two-layer ReLU neural network with width $m= o(\gamma^{-1}{\log n})$ it holds with constant probability over the random initialization of $W_0$ that for any weights $V \in \mathbb{R}^{m \times d}$ there exists at least one index $i \in [n]$ such that $y_i \langle \nabla f_i(W_0), V \rangle~\leq~0$.
\end{theorem}
As pointed out in \cite{jt20} this does not necessarily mean that gradient descent cannot achieve arbitrarily small training error for lower width, but the NTK analysis fails in this case.

An even stronger assumption is that we must rely on the finite dimensional separator $\bar U$ in the analysis of \citet{jt20} that mimics the NTK separator $\bar v$ in the RKHS achieving a margin of $\gamma>0$. In this case we can show that our upper bound is indeed tight, i.e., for this natural choice of $\bar U$ and the necessity of a union bound over $n$ points, we have $m=\Theta({\gamma^{-2}}\log n)$, which follows from the following lemma. 
\begin{lemma}[informal version of Lemma \ref{lemweakupbo}]\label{LB:weak_informal}

There exists a data set in $2$-dimensional space, such that for the two-layer ReLU neural network with parameter matrix $\bar U$ and width $m = o({\gamma^{-2}}\log n)$, with constant probability there exists an $i\in[n]$ such that $y_i\langle \nabla f_i(W_0), \bar U \rangle \leq 0$.
\end{lemma}
In fact, this is also the only place in our improved analysis where the width $m$ depends on $\operatorname{poly}(\log n, {1}/{\gamma})$; everywhere else it only depends on $\log({1}/{\epsilon})$. Our linear upper bound for the $2$-dimensional space gets around this lower bound by defining a different $\bar U$ in Lemma \ref{d2lem2.3}:

\begin{lemma}[follows directly using Lemma \ref{d2lem2.3} in the analysis of Theorem \ref{thm:logwidth:main_informal}]\label{lem:2dupper_informal}
Given $n$ labeled data points in $2$-dimensional space, consider a two-layer ReLU neural network with width $m = \Omega( {\gamma^{-1}}{{\log n}} )$. Starting from a coupled initialization (Def. \ref{def:initialization}), for arbitrary accuracy $\epsilon \in (0,1)$ , we can ensure the cross-entropy (logistic) training loss is less than $\epsilon$ when running gradient descent for $T= O( {\epsilon^{-1} \gamma^{-2}}{\log^2(1/\epsilon)} )$ iterations.
\end{lemma}

However, the construction in Lemma \ref{d2lem2.3} / \ref{lem:2dupper_informal} uses a net argument of size $({1}/{\gamma})^{d-1}$ to discretize the points on the sphere, and that -- already in $3$ dimensions -- matches the quadratic general upper bound and becomes worse in higher dimensions. It thus remains an open question whether there are better separators in dimensions $d\geq 3$ or if the quadratic lower bound is indeed tight. We also present a candidate hard instance, for which we conjecture that it has an $\Omega(\gamma^{-2})$ lower bound, up to logarithmic factors, for any algorithm; see Sec. \ref{sec:lb_informal}.

Next, we move on to the analysis of the squared loss. We first state our assumption that is standard in the literature on the width of neural networks, and is necessary to guarantee the existence of an arbitrarily accurate parameterization \cite{dzps19,sy19}.

\begin{ass}\label{ass:mineigval}
Let $K$ be the NTK kernel matrix where for each $i,j\in[n]$ we have that $K_{i,j}$ equals
\begin{align*}
    K(x_i,x_j) = \E_{w \sim {\cal N}(0,I_d) } [ x_i^\top x_j \mathbf{1}[\langle x_i,w \rangle > 0,\langle x_j,w \rangle > 0] ] .
\end{align*}
We assume in the following that the smallest eigenvalue $\lambda(K)$ of $K$ satisfies $\lambda(K)>\lambda$, for some value $\lambda>0$.
\end{ass}

We state our main result for squared loss as follows:
\begin{theorem}[informal version of Theorem~\ref{thm:main_formal}]\label{thm:main_informal}
Given $n$ input data points in $d$-dimensional space, consider a two-layer neural network with width $m = \Omega( \lambda^{-2} n^2 )$. Starting from a coupled initialization (Def. \ref{def:initialization}) and for any accuracy $\epsilon \in (0,1)$, the squared training loss is smaller than $\epsilon$ after $T= O(\lambda^{-2} n^2 \log(1/\epsilon) )$ iterations of gradient descent.
\end{theorem}

%% file: arxivtech.tex
\section{Technical Overview}\label{sec:techoverview}

\subsection{Logarithmic Width for Logistic Loss, Upper Bound}
The work of \cite{jt20} shows that we can bound the actual logistic loss averaged over $T$ gradient descent iterations $W_t, t \in [T]$ using any reference parameterization $\bar W$ in the following NTK bound:
\begin{align}\label{eqn:start_informal}
\frac{1}{T} \sum_{t = 1}^T L(W_t)
\leq \frac{1}{T} \Vert W_{0} - \bar{W} \Vert_F^2 + \frac{2}{T} \sum_{t =1}^T L^{(t)}(\bar{W}), 
\end{align}
where $L^{(t)}(\bar W):=\sum_{i=1}^n \ell\left(y_i, \langle \nabla f_i(W_t), \bar W \rangle \right)$.
It seems very natural to choose $\bar W = W_0 + \rho \bar U$ where $\bar U$ is a reasonably good separator for the NTK points with bounded norm $\|\bar U\|_F\leq 1$, meaning that for all $i$ it holds that $y_i\langle \nabla f_i(W_0), \bar U \rangle = \Omega(\gamma)$. It thus has the same margin as in the infinite case up to constants. This already implies that the first term of Eq.~\eqref{eqn:start_informal} is sufficiently small when we choose roughly $T=\rho^2/\epsilon$ iterations. Now, in order to bound the second term, \citet{jt20} propose to show $L^{(t)}(\bar{W})\leq \epsilon$ for every $t\leq T$, which is implied if for each index $i\in [n]$ we have that
\begin{align*}
y_i\langle &\nabla f_i(W_t), \bar W \rangle = y_i\langle \nabla f_i(W_0), W_0 \rangle + y_i\langle \nabla f_i(W_t) - \nabla f_i(W_0), W_0 \rangle + \rho y_i\langle \nabla f_i(W_t), \bar U \rangle
\end{align*}
is sufficiently \emph{large}. Here, we can leverage the coupled initialization scheme (Def. \ref{def:initialization}) to prove Theorem \ref{thm:logwidth:main_informal}:
bounding the first term for random Gaussian parameters, this results in roughly the value $\sqrt{\log n}$, but now since for each Gaussian vector there is another identical Gaussian with opposite signs, those simply cancel and we have $y_i\langle \nabla f_i(W_0), W_0 \rangle = 0$ in the initial state.

To bound the second term, the previous analysis \cite{jt20} relied on a proper scaling with respect to the parameters $m, \rho,$ and $\gamma$, where the requirement that $m\geq \rho^2/\gamma^6$ led to a bound of roughly $m\geq\gamma^{-8} \log n$. Using the coupled initialization, however, the terms again cancel in such a way that the scaling does not matter, and in particular does not need to be balanced among the variables. Another crucial insight is that the gradient is entirely independent of the scale of the parameter vectors in $W_0$. This implies $\nabla f_i(W_t) = \nabla f_i(W_0)$ and thus $y_i\langle \nabla f_i(W_t) - \nabla f_i(W_0), W_0 \rangle = 0$ again, notably without any implications for the width of the neural network!

Indeed, the only place in the analysis where the width is constrained by roughly $m\geq \gamma^{-2}\log n$ occurs when bounding the contribution of the third term by $\rho y_i\langle \nabla f_i(W_t), \bar U \rangle = \rho y_i\langle \nabla f_i(W_0), \bar U \rangle = \Omega(\rho \gamma)$. This is done exactly as in \cite{jt20} by a Hoeffding bound to relate the separation margin of the finite subsample to the separation margin of the infinite width case, i.e.,
\begin{align}
    y_i \frac{1}{m} &\sum\limits_{j=1}^{m} \langle \bar v(z_j), x_i \rangle \mathbf{1}[\langle z_j, x_i \rangle > 0] 
    \approx y_i \int \langle \bar v(z), x_i \rangle \mathbf{1}[\langle z, x_i \rangle > 0]~\d \mu_N(z) \geq \gamma \label{eqn:hoeffding_informal}
\end{align}
followed by a union bound over all $n$ input points.

The special and natural choice of $\bar U$ such that $\bar u_j = a_j\bar v(z_{j})/\sqrt{m}$ yields Eq. \eqref{eqn:hoeffding_informal} above, where notably the LHS equals the RHS in expectation. We will discuss this particular choice again in our lower bounds section \ref{sec:lb_informal}.

\subsection{Logarithmic Width for Logistic Loss, Lower Bounds}\label{sec:lb_informal}
Our assumption on the separation margin is formally defined in Section \ref{sec:examples} where we also give several examples and useful lemmas to bound $\gamma$.
Our lower bounds in Section \ref{sec:LB_logwidth} are based on the following hard instance in $2$ dimensions. The points are equally spaced and with alternating labels on the unit circle.

Formally, let $n$ be divisible by $4$. We define the \emph{alternating points on the circle} data set to be $X=\{x_k := \left(\cos\left(\frac{2k \pi}{n}\right),  \sin\left(\frac{2k \pi}{n}\right) \right) \mid k\in [n]\} \subset \R^2$, and we put $y_k=(-1)^k$ for each $k\in[n]$.

A natural choice for $\bar v$ would send any $z \in \mathbb{R}^d$ to its closest point in our data set $X$, multiplied by its label. However, applying Lemma \ref{psetlem} gives us the following improved mapping, which is even optimal by Lemma \ref{gammabound}: note that for any $z \in \mathbb{R}^d$ that is not collinear with any input point, there exists a unique $i_z$ such that $z\in \Cone(\{x_{i_z}, x_{i_z+1}\})$. Instead of mapping to the closest input point, in what follows, we map to a point that is nearly orthogonal to $z$, \begin{align*}
r_z := \frac{x_{i_z}y_{i_z}+x_{i_z+1}y_{i_z+1}}{\| x_{i_z}-x_{i_z+1} \|_2}.
\end{align*}
More precisely we define $\bar{v} : \mathbb{R}^d \rightarrow \mathbb{R}^d$ by
\[
    \bar{v}(z) = 
  \begin{cases}
    0 &\text{, if } \exists x_i\in X, \tau\geq 0: z=\tau x_i\\
    (-1)^{n/4+1}r_z &\text{, otherwise}.
  \end{cases}
\]

\begin{figure}[ht!]
\begin{center}
\begin{tabular}{c}
\includegraphics[width=.5\linewidth]{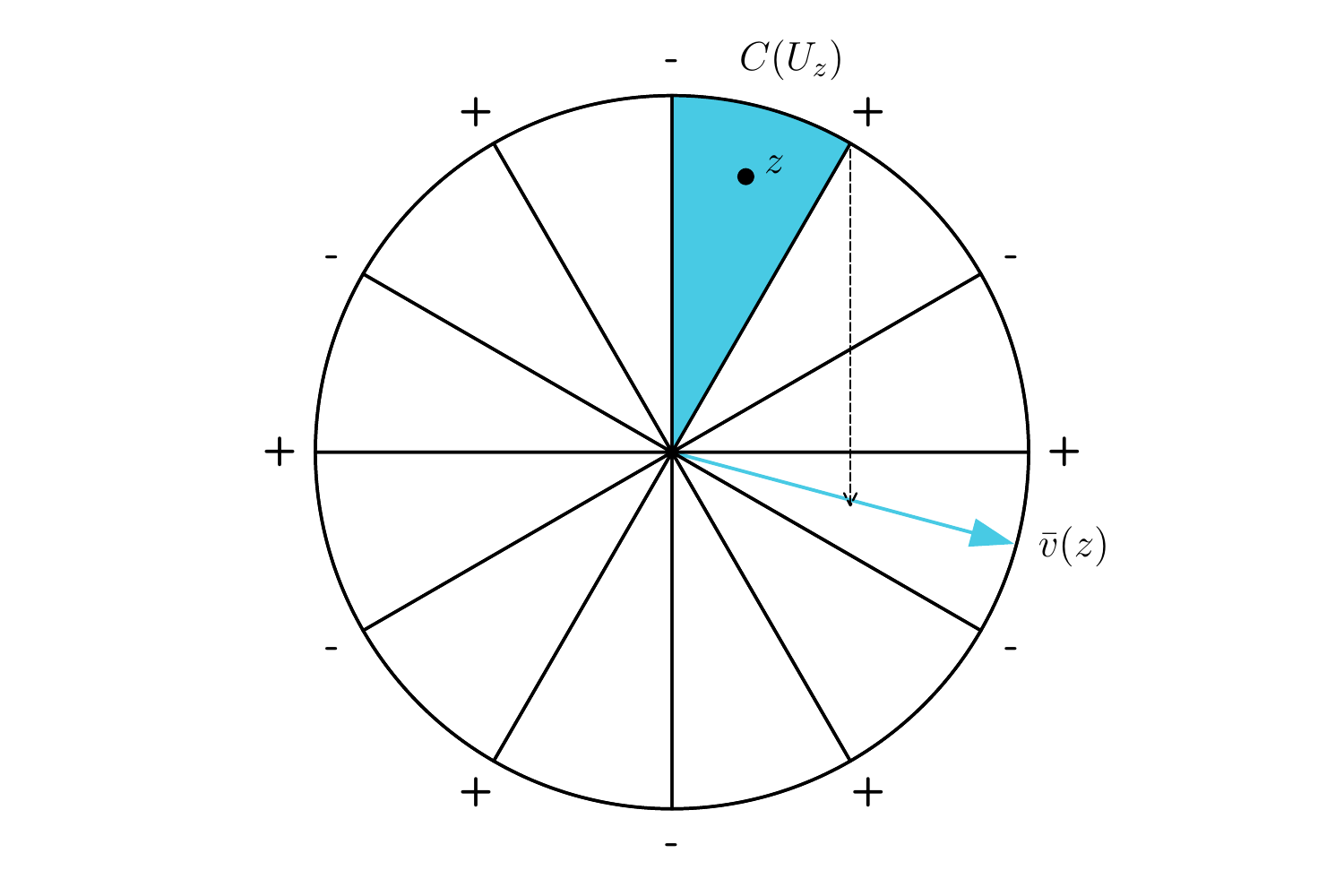}
\end{tabular}
\caption{The picture shows how $\bar{v}(z)$ is constructed: we subtract the vector $x_3$ which is labeled $-1$ from the vector $x_2$ which is labeled $1$. We obtain $r_z$ after rescaling to unit norm. Since $n/4=3$ is odd we have $\bar{v}(z)=r_z$. 
}
\label{fig:altcircle_main}
\end{center}
\end{figure}

See Fig. \ref{fig:altcircle_main} for an illustration. We show in Lemma \ref{Prop5.3} that $\gamma_{\bar v}= \gamma(X,Y)=\Omega(n^{-1})$ and consequently $n=\Omega(\gamma^{-1})$. Now we can derive our lower bounds under increasingly stronger assumptions as follows:

For the first unconditional bound, Theorem \ref{LB:uncon_informal}, we map the input points on the unit circle by contraction to the unit $\ell_1$ ball and note that by doing so, the labels remain in alternating order. Next we note that the output function $f$ of our network restricted to the $\ell_1$ ball is a piecewise linear function of $x$ and thus its gradient $\frac{\partial f}{\partial x}$ can only change in the vertices of the ball or where $x$ is orthogonal to one of the parameter vectors $w_s$, i.e., at most $O(m)$ times. Now consider any triple of consecutive points. Since they have alternating labels, the gradient needs to change at least once for each triple. Consequently $m=\Omega(n)=\Omega(\gamma^{-1})$, improving the $\Omega(\gamma^{-1/2})$ conditional lower bound in \cite{jt20}. We remark that since the argument only depends on the output function but not on the loss function, Lemma \ref{Prop5.4'} yields an $m=\Omega(n)$ lower bound for any loss function, in particular also for squared loss.

Now consider any argument that relies on an NTK analysis, where the fist step is to show that for our choice of the width $m$, we have separability at initialization; this is how the argument in \cite{jt20} proceeds. Formally, assume that there exists $W\in\mathbb{R}^{m\times d}$ such that for all $i\in[n]$, we have $y_i\langle \nabla f_i(W_0), W \rangle > 0$. This condition enables us to show an improved lower bound, Theorem \ref{LB:NTK_informal}, as follows. Partition our data set into tuples, each consisting of four consecutive points. Now consider the event that there exists a point $x_i$ such that all parameter vectors $w_s$ satisfy
\begin{align*}
    \mathbf{1}[ \langle x_i, w_s \rangle >0 ]&=\mathbf{1}[ \langle x_{i+1}, w_s \rangle >0 ]
    =\mathbf{1}[ \langle x_{i+2}, w_s \rangle >0 ]=\mathbf{1}[ \langle x_{i+3}, w_s \rangle >0 ],
\end{align*}
which implies that at least one of the points in $\{x_i,x_{i+2},x_{i+3},x_{i+4}\}$ is misclassified. To avoid this, it means that our initialization necessarily needs to include, for each $i$ divisible by $4$, a vector $w_s$ that hits the areas orthogonal to the cones separating two points out of the quadruple. There are $O(n)$ quadruples to hit, each succeeding with probability $\Omega(n^{-1})$ with respect to the Gaussian measure. This is exactly the coupon collector's problem, where the coupons are the quadruples, and it is known \cite{er61} that $m=\Omega(n\log n)=\Omega(\gamma^{-1}\log n)$ are necessary to collect all $O(n)$ items with constant probability, which yields our improved lower bound for this style of analysis. One thus needs a different approach beyond NTK to break this barrier, cf. \cite{jt20}.

For the upper bound, Theorem \ref{thm:logwidth:main_informal}, we further note that the existence of an NTK separator $W$ as above is not sufficient, i.e., we need to \emph{construct} a separator $\bar U$ satisfying the separability condition. Moreover, to achieve a reasonable bound in terms of the margin parameter $\gamma^{-1}$ we also need that $\bar U$ achieves a separation margin of $\Omega(\gamma)$. To do so, it seems natural to construct $\bar U$ such that $\bar u_j = a_j\bar v(w_j)/\sqrt{m}$ for all $j\in[m]$. Indeed, this is the most natural choice because the resulting separation margin in the finite case is exactly the empirical mean estimator for the infinite case and thus standard concentration bounds (Hoeffing's bound) yield the necessary proximity of the margins between the two cases, cf. Eq. \eqref{eqn:hoeffding_informal}. Let us assume we fix this choice of $\bar U$. This condition enables us to prove a quadratic lower bound, Lemma \ref{LB:weak_informal}, which shows that our analysis of the upper bound is actually tight: for our alternating points on a circle example, the summands have high variance, and therefore Hoeffding's bound is tight by a matching lower bound \cite{Feller43}. Consequently, $m=\Omega(\gamma^2\log n)$, and one would need a different definition of $\bar U$ and a non-standard estimation of $\gamma$ to break this barrier.

Finally, we conjecture that the quadratic upper bound is actually tight in general. Specifically, we conjecture the following: take $n=1/\gamma^2$ random points on the sphere in $\Theta(\log(1/\gamma))$ dimensions and assign random labels $y_i\in \{-1,1\}$. Then the NTK margin is $\Omega(\gamma)$.

If the conjecture is true, we obtain an $\Omega(1/\gamma^2)$ lower bound for $m$, up to logarithmic factors.\footnote{This might be confusing, since we argued before that such data is particularly \emph{mild} for the squared loss function. This may be due to the different loss functions, but regardless, it does not contradict the $\tilde{O}(n/d)$ bound of \cite{Daniely20} for the same data distribution in the squared loss regime, since $n/d = \Theta( 1/(\gamma^2 \log(1/\gamma)) )$.} Indeed, we can round the weights to the nearest vectors in a net of size $\operatorname{poly}(1/\gamma)^{O(\log(1/\gamma))}$, which only changes $\gamma$ by a constant factor. Then, if we could classify with zero error, we would encode $n=1/\gamma^2$ random labels using $m\log^{O(1)}(1/\gamma)$ bits, which implies $m \geq \Omega(1/(\gamma^2 \log^{O(1)}(1/\gamma)))$. We note that \citet{jt20} gave an $O(1/\sqrt{n})$ upper bound for the margin of any data with random labels, but we would need a matching $\Omega(1/\sqrt{n})$ lower bound for this instance in order for this argument to work.

\subsection{Polynomial Width for Squared Loss}
The high level intuition of our proof of Theorem \ref{thm:main_informal} is to recursively prove the following: (1) the weight matrix does not change much, and (2) given that the weight matrix does not change much, the prediction error, measured by the squared loss, decays exponentially.

Given (1) we prove (2) as follows. The intuition is that the kernel matrix does not change much, since the weights do not change much, and it is close to the initial value of the kernel matrix, which is in turn close to the NTK matrix, involving the entire Gaussian distribution rather than our finite sample. The NTK matrix has a lower bound on its minimum eigenvalue by Assumption~\ref{ass:mineigval}. Thus, the prediction loss decays exponentially.

Given (2) we prove (1) as follows. Since the prediction error decays exponentially, one can show that the change in weights is upper bounded by the prediction loss, and thus the change in weights also decays exponentially and the total change is small.

First, we show a concentration lemma for initialization:
\begin{lemma}[Informal version of Lemma~\ref{lem:3.1}]
\label{lem:3.1:intro}
Let $m = m_0 B$.
Let $\{ w_1, w_2, \ldots, w_m \} \subset\R^d$ denote a collection of vectors constructed as in Definition~\ref{def:initialization}.
We define $H^{\cts}, H^{\dis} \in \R^{n \times n}$ as follows
\begin{align*}
&H^{\cts}_{i,j} :=  \E_{w \sim \N(0,I)} \left[ x_i^\top x_j {\bf 1}_{ w^\top x_i \geq 0, w^\top x_j \geq 0 } \right] ,  \qquad H^{\dis}_{i,j} :=  \frac{1}{m} \sum_{r=1}^m \left[ x_i^\top x_j {\bf 1}_{ w_r^\top x_i \geq 0, w_r^\top x_j \geq 0 } \right].
\end{align*}
Let $\lambda = \lambda_{\min} (H^{\cts}) $. If $m_0 = \Omega( \lambda^{-2} n^2\log (nB/\delta) )$, we have that  
\begin{align*}
\| H^{\dis} - H^{\cts} \|_F \leq \frac{ \lambda }{4}, \mathrm{~and~} \lambda_{\min} ( H^{\dis} ) \geq \frac{3}{4} \lambda
\end{align*}
holds with probability at least $1-\delta$. 
\end{lemma}
Second, we can show a perturbation bound for random weights.
\begin{lemma}[Informal version of Lemma~\ref{lem:3.2}]\label{lem:3.2:intro}
Let $R \in (0,1)$. Let $\{w_1, w_2, \ldots, w_m\}$ denote a collection of weight vectors constructed as in Definition~\ref{def:initialization}. 
For any set of weight vectors $\wt{w}_1, \ldots, \wt{w}_m \in \R^d$ that satisfy that for any $r\in [m]$, $\| \wt{w}_r - w_r \|_2 \leq R$, consider the map $H : \R^{m \times d} \rightarrow \R^{n \times n}$ defined by  $ H(w)_{i,j} :=  \frac{1}{m} x_i^\top x_j \sum_{r=1}^m {\bf 1}_{ \wt{w}_r^\top x_i \geq 0, \wt{w}_r^\top x_j \geq 0 } $. Then we have that  $\| H (w) - H(\wt{w}) \|_F < 2 n R$ holds with probability at least $1-n^2 \cdot B \cdot \exp(-m_0 R /10)$.
\end{lemma}

Next, we have the following lemma (see Section \ref{sec:missing_proof} for a formal proof) stating that the weights should not change too much. 
Note that the lemma is a variation of Corollary 4.1 in \cite{dzps19}.
\begin{lemma}\label{lem:4.1:intro}
If Eq. \eqref{eq:quartic_condition} 
holds for $i = 0, \ldots, k$, then we have for all $r\in [m]$
\begin{align*}
\| w_r(t+1) - w_r(0) \|_2 \leq \frac{ 4 \sqrt{n} \| y - u (0) \|_2 }{ m \lambda } := D.
\end{align*}
\end{lemma}

Next, we calculate the difference of predictions between two consecutive iterations, analogous to the $\frac{\d u_i(t)}{ \d t }$ term in Fact \ref{fact:dudt}.
For each $i \in [n]$, we have
{
\begin{align*}
u_i(t+1) - u_i(t) 
=  \sum_{r=1}^m &a_r \cdot \left( \phi( w_r(t+1)^\top x_i ) 
-\, \phi(w_r(t)^\top x_i ) \right)
=  \sum_{r=1}^m a_r \cdot z_{i,r} .
\end{align*}
}
where
\begin{align*}
z_{i,r} :=  \phi \left( \Big( w_r(t) - \eta \frac{ \partial L( W(t) ) }{ \partial w_r(t) } \Big)^\top x_i \right) - \phi ( w_r(t)^\top x_i ) .
\end{align*}

Here we divide the right hand side into two parts. First, $v_{1,i}$ represents the terms for which the pattern does not change, while $v_{2,i}$ represents the terms for which the pattern may change. For each $i \in [n]$,
we define the set $S_i\subset [m]$ as
\begin{align*}
    S_i:= ~ \{r\in [m]:\forall  w\in \mathbb{R}^d &\text{ s.t. } \|w-w_r(0)\|_2\leq R, 
    \text{ and }\mathbf{1}_{w_r(0)^\top x_i\geq 0}= \mathbf{1}_{w^\top x_i\geq 0}\}.
\end{align*}
Then we define $v_{1,i}$ and $v_{2,i}$ as follows
{
\begin{align*}
v_{1,i} : =  \sum_{r \in S_i} a_r z_{i,r}, \qquad
v_{2,i} : =  \sum_{r \in \ov{S}_i} a_r z_{i,r}.
\end{align*} 
}

Define $H$ and $H^{\bot} \in \R^{n \times n}$ as
\begin{align}
H(t)_{i,j} := & ~ \frac{1}{m} \sum_{r=1}^m x_i^\top x_j {\bf 1}_{ w_r(t)^\top x_i \geq 0, w_r(t)^\top x_j \geq 0 } , \\
H(t)^{\bot}_{i,j} := & ~ \frac{1}{m} \sum_{r\in \ov{S}_i} x_i^\top x_j {\bf 1}_{ w_r(t)^\top x_i \geq 0, w_r(t)^\top x_j \geq 0 } \label{eq:def_H(k)^bot},
\end{align}
and
 \begin{align*}
 &C_1 := -2 \eta (y - u(t))^\top H(t) ( y - u(t) ), \qquad C_2 :=  + 2 \eta ( y - u(t) )^\top H(t)^{\bot} ( y - u(t) ) , \\
 &C_3 := - 2 ( y - u(t) )^\top v_2 , \qquad\qquad\qquad\qquad\! C_4 :=  \| u (t+1) - u(t) \|_2^2 . 
 \end{align*}
Then we have that (see Section \ref{sec:missing_proof} for a formal proof)
\begin{claim}
$\| y - u(t+1) \|_2^2 = \| y - u(t) \|_2^2 + C_1 + C_2 + C_3 + C_4.$
\end{claim}

Applying Claim~\ref{cla:C1}, \ref{cla:C2}, \ref{cla:C3} and \ref{cla:C4} with the appropriate choice of parameters, we can show that the $\ell_2$ norm shrinks in each iteration $t$:
$
  \| y - u(t+1) \|_2^2 
\leq \| y - u(t) \|_2^2 
  \cdot \alpha
$, 
where $\alpha=( 1 - m \eta \lambda + 8 m \eta n R  + 8 m \eta n R  + m^2 \eta^2 n^2 )$.

%% file: arxivdiscuss.tex
\section{Discussion}\label{sec:discuss}

We present a novel worst case analysis using an initialization scheme for neural networks involving coupled weights. This technique is versatile and can be applied in many different settings. 
We give an improved analysis based on this technique to reduce the parameterization required to show the convergence of $2$-layer neural networks with ReLU activations in the under-parameterized regime for the logistic loss to $m=O(\gamma^{-2}\log n)$, which significantly improves the prior $O(\gamma^{-8} \log n)$ bound. We further introduce a new unconditional lower bound of $m=\Omega(\gamma^{-1})$ as well as conditional bounds to narrow the gap in this regime.
We also reduce the amount of  over-parameterization required for the standard squared loss function to roughly $m = O(\lambda^{-2}n^2)$, improving the prior $O(\lambda^{-4} n^4)$ bound, and coming closer to the $\Omega(n)$ lower bound.
We believe this is a significant theoretical advance towards explaining the behavior of $2$-layer neural networks in different settings. It is an intriguing open question to close the gaps between upper and lower bounds in both the under and over-parameterized settings. We note that the quadratic dependencies arise for several fundamental reasons in our analysis, and are already required to show a large minimum eigenvalue $\lambda$ of our kernel matrix, or a large separation margin $\gamma$ at initialization. In the latter case we have provided partial evidence of optimality by showing that our analysis has an $m=\Omega(\gamma^{-2}\log n)$ lower bound, as well as a candidate hard instance for any possible algorithm. Another future direction is to extend our results to more than two layers, which may be possible by increasing $m$ by a $\operatorname{poly}(L)$ factor, where $L$ denotes the network depth \cite{Chen19}. We note that this also has not been done in earlier work \cite{jt20}.

%% file: arxivprobability.tex
\section{Probability Tools}\label{sec:probability_tools}

In this section we introduce the probability tools that we use in our proofs.
Lemma \ref{lem:chernoff}, \ref{lem:hoeffding} and \ref{lem:bernstein} concern tail bounds for random scalar variables.
Lemma \ref{lem:anti_gaussian} concerns the cumulative density function of the Gaussian distribution.
Finally,
Lemma \ref{lem:matrix_bernstein} concerns a concentration result for 
random matrices.

\begin{lemma}[Chernoff bound \cite{c52}]\label{lem:chernoff}
Let $X = \sum_{i=1}^n X_i$, where $X_i=1$ with probability $p_i$ and $X_i = 0$ with probability $1-p_i$, and all $X_i$ are independent. Let $\mu = \E[X] = \sum_{i=1}^n p_i$. Then \\
1. $ \Pr[ X \geq (1+\delta) \mu ] \leq \exp ( - \delta^2 \mu / 3 ) $, $\forall \delta > 0$ ; \\
2. $ \Pr[ X \leq (1-\delta) \mu ] \leq \exp ( - \delta^2 \mu / 2 ) $, $\forall 0 < \delta < 1$. 
\end{lemma}

\begin{lemma}[Hoeffding bound \cite{h63}]\label{lem:hoeffding}
Let $X_1, \cdots, X_n$ denote $n$ independent bounded variables in $[a_i,b_i]$. Let $X= \sum_{i=1}^n X_i$. Then we have
\begin{align*}
\Pr[ | X - \E[X] | \geq t ] \leq 2\exp \left( - \frac{2t^2}{ \sum_{i=1}^n (b_i - a_i)^2 } \right).
\end{align*}
\end{lemma}

\begin{lemma}[Bernstein's inequality \cite{b24}]\label{lem:bernstein}
Let $X_1, \cdots, X_n$ be independent zero-mean random variables. Suppose that $|X_i| \leq M$ almost surely, for all $i$. Then, for all positive $t$,
\begin{align*}
\Pr \left[ \sum_{i=1}^n X_i > t \right] \leq \exp \left( - \frac{ t^2/2 }{ \sum_{j=1}^n \E[X_j^2]  + M t /3 } \right).
\end{align*}
\end{lemma}

\begin{lemma}[Anti-concentration of the Gaussian distribution]\label{lem:anti_gaussian}
Let $X\sim {\N}(0,\sigma^2)$,
that is,
the probability density function of $X$ is given by $\phi(x)=\frac 1 {\sqrt{2\pi\sigma^2}}e^{-\frac {x^2} {2\sigma^2} }$.
Then
\begin{align*}
    \Pr[|X|\leq t]\in \left( \frac 2 3\frac t \sigma, \frac 4 5\frac t \sigma \right).
\end{align*}
\end{lemma}

\begin{lemma}[Matrix Bernstein, Theorem 6.1.1 in \cite{t15}]\label{lem:matrix_bernstein}
Consider a finite sequence $\{ X_1, \cdots, X_m \} \subset \R^{n_1 \times n_2}$ of independent, random matrices with common dimension $n_1 \times n_2$. Assume that
\begin{align*}
\E[ X_i ] = 0, \forall i \in [m] ~~~ \mathrm{and}~~~ \| X_i \| \leq M, \forall i \in [m] .
\end{align*}
Let $Z = \sum_{i=1}^m X_i$. Let $\mathrm{Var}[Z]$ be the matrix variance statistic of the sum:
\begin{align*}
\mathrm{Var} [Z] = \max \left\{ \Big\| \sum_{i=1}^m \E[ X_i X_i^\top ] \Big\| , \Big\| \sum_{i=1}^m \E [ X_i^\top X_i ] \Big\| \right\}.
\end{align*}
Then 
\begin{align*}
\E[ \| Z \| ] \leq ( 2 \mathrm{Var} [Z] \cdot \log (n_1 + n_2) )^{1/2} +  M \cdot \log (n_1 + n_2) / 3.
\end{align*}

Furthermore, for all $t \geq 0$,
\begin{align*}
\Pr[ \| Z \| \geq t ] \leq (n_1 + n_2) \cdot \exp \left( - \frac{t^2/2}{ \mathrm{Var} [Z] + M t /3 }  \right)  .
\end{align*}
\end{lemma}

\begin{lemma}[\cite{Feller43}]\label{lemweakupbohelp}
Let $Z$ be a sum of independent random variables, each attaining values in $[0,1]$, and let $\sigma=\sqrt{\Var(Z)}\geq 200$. Then for all $t \in [0, \frac{\sigma^2}{100}]$ we have
\[ \Pr[ X\geq \E[X]+t] \geq c \cdot \exp( -t^2/(3\sigma^2) ) \]
where $c>0$ is some fixed constant.
\end{lemma}

%% file: arxivlogwidth.tex
\newpage
\section{Preliminaries for log width under logistic loss} \label{sec:logwidth:start}

We consider a set of data points $x_1, \dots , x_n \in \mathbb{R}^d$ with $\| x_i\|_2 =1$ and labels $y_1 , \dots , y_n \in \{-1, 1\}$.
The two layer network is parameterized by $m \in \mathbb{N}, a \in \mathbb{R}^m$ and $W \in \mathbb{R}^{m \times d}$ as follows:
we set the output function
\begin{align*}
f(x, W, a)=  \frac{1}{\sqrt{m}} \sum_{s=1}^m a_s \phi\left(\langle w_s, x \rangle \right),
\end{align*}
which is scaled by a factor $1/\sqrt{m}$ compared to the presentation in the main body to simplify notation, and to be more closely comparable to \cite{jt20}. The changed initialization yields initial output of $0$, independent of the normalization, and thus, consistent with the introduction, we could as well omit the normalization \emph{here} and instead use it only in the learning rate. The main improvement of the network width comes from the fact that the learning rate is no compromise between the right normalization in the initial state and the appropriate progress in the gradient iterations, but can be adjusted to ensure the latter independent of the former.
In the output function, $\phi(v)=\max\{0, v\}$ denotes the ReLU function for $v\in  \mathbb{R}$.
To simplify notation we set $f_i(W)=f(x_i, W, a)$.
Further we set $\ell(v)=\ln(1+\exp(-v))$ to be the logistic loss function.
We use a random initialization $W_0, a_0$ given in Definition \ref{def:initialization}.
Our goal is to minimize the empirical loss of $W$ given by
\[ R(W)=\frac{1}{n}\sum_{i=1}^n \ell\left(y_if_i(W)\right). \]
To accomplish this, we use a standard gradient descent algorithm.
More precisely for $t\geq 0$ we set
\begin{align*}
W_{t+1}=W_t-\eta \nabla R(W_t)
\end{align*}
for some step size $\eta$.
Further, it holds that
\begin{align*}
\nabla R(W)=\frac{1}{n}\sum_{i=1}^n y_i \nabla f_i(W) \ell'\left(y_if_i(W)\right).
\end{align*}
Moreover, we use the following notation
\begin{align*}
f_i^{(t)}(W):=\langle \nabla f_i(W_t), W \rangle
\end{align*}
and
\begin{align*}
R^{(t)}(W):=\sum_{i=1}^n \ell\left(y_i f_i^{(t)}(W) \right).
\end{align*}
Note that $\frac{\partial f_i(W)}{\partial w_s}=\frac{1}{\sqrt{m}} a_s \mathbf{1}[\langle w_s, x_i \rangle >0]x_i$.
In particular the gradient is independent of $\| w_s \|_2 $, which will be crucial in our improved analysis.

\section{Main assumption and examples}
\label{sec:examples}
\subsection{Main assumption}

Here, we define the parameter $\gamma>0$ which was also used in \cite{jt20}.
Intuitively, $\gamma$ determines the separation margin of the NTK.
Let $B=B^{d}=\lbrace x \in \mathbb{R}^d ~|~ \| x \|_2\leq 1 \rbrace$ be the unit ball in $d$ dimensions.
We set $\mathcal{F}_B$ to be the set of functions $f$ mapping from $\dom(f)=\mathbb{R}^d$ to $ \range(f) = B $.
Let $\mu_{\mathcal{N}}$ denote the Gaussian measure on $\mathbb{R}^d$, specified by the Gaussian density with respect to the Lebesgue measure on $\mathbb{R}^d$.

\begin{mydef}\label{def1}
Given a data set $(X, Y) \in \mathbb{R}^{n \times d} \times \mathbb{R}^n$ and a map $\bar{v} \in \mathcal{F}_B$ we set
\begin{align*}
\gamma_{\bar{v}}=\gamma_{\bar{v}}(X,Y):=\min_{i \in [n]} y_i\int \langle \bar{v}(z), x_i \rangle \mathbf{1}[ \langle x_i, z \rangle >0 ]~ \d \mu_{\mathcal{N}}(z) .
\end{align*}
We say that $\bar{v}$ is optimal if $\gamma_{\bar{v}}=\gamma(X, Y):= \max_{\bar{v}' \in \mathcal{F}_B}\gamma_{\bar{v}'}$.
\end{mydef}

We note that $\max_{\bar{v}' \in \mathcal{F}_B}\gamma_{\bar{v}'}$ always exists since $ \mathcal{F}_B$ is a set of bounded functions on a compact subset of $\mathbb{R}^d$.
We make the following assumption, which is also used in \cite{jt20}:

\begin{ass}\label{assu}
It holds that $\gamma=\gamma(X, Y)>0$.
\end{ass}

Before we prove our main results we show some properties of $\bar{v}$ to develop a better understanding of our assumption.
The following lemma shows that the integral can be viewed as a finite sum over certain cones in $\mathbb{R}^d$.
Given $U \subseteq \{1, 2, \dots, n\}=[n]$ we define the cone 
\begin{align*}
C(U):=\{ x \in \mathbb{R}^d ~|~ \langle x, x_i \rangle >0 \text{ if and only if $i \in U$}  \}.
\end{align*}

Note that $C(\emptyset)= \{ x \in \mathbb{R}^d ~|~ \langle x, x_i \rangle \leq 0 \text{ for all $i \in [n]$}\}$ and that $\mathbb{R}^d=\dot{\bigcup}_{U \subseteq [n] }C(U)$.
Further we set $P(U)$ to be the probability that a random Gaussian is an element of $C(U)$ and $P_U$ to be the probability measure of random Gaussians $z\sim \mathcal N(0,I)$ restricted to the event that $z\in C(U)$.
The following lemma shows that we do not have to consider each mapping in $\mathcal{F}_B$ but it suffices to focus on a specific subset.
More precisely we can assume that $\bar{v}$ is constant on the cones $C(U)$.
In particular this means we can assume $\bar{v}(z)=\bar{v}(c z) $ for any $z \in \mathbb{R}^d$ and scalar $c>0$ and that $\bar{v}$ is locally constant.

\begin{lem}\label{psetlem}
Let $\bar{v}\in \mathcal{F}_B$.
Then there exists $\bar{v}'$ such that $\gamma_{\bar{v}'}=\gamma_{\bar{v}}$ and $\bar{v}'$ is constant on $C(U)$ for any $U \subseteq [n]$.
\end{lem}

\begin{proof}
Observe that for any distinct $U, U'\subseteq [n]$ the cones $C(U)$ and $C(U')$ are disjoint since for any $x\in \mathbb{R}^d$ the cone $C(U_x)$ containing $x$ is given by $U_x=\{ i \in [n] ~|~ \langle x, x_i \rangle >0 \}$.
Further we have that $ \bigcup_{U \subseteq [n]} C(U) = \mathbb{R}^d$ since any $x \in \mathbb{R}^d$ is included in some $C(U_x)$.
Thus for any $i \in [n]$ we have
\begin{align*}
y_i\int \langle \bar{v}(z), x_i \rangle \mathbf{1}[ \langle x_i, z \rangle >0 ]~ \d \mu_{\mathcal{N}}(z) 
&= y_i\sum_{U \subseteq [n]}P(U) \int \langle \bar{v}(z), x_i \rangle \mathbf{1}[ \langle x_i, z \rangle >0 ]~ \d P_U(z)\\
&=  y_i\sum_{U \subseteq [n], i \in U}P(U) \int \langle \bar{v}(z), x_i \rangle~ \d P_U(z) \\
&=  y_i\sum_{U \subseteq [n], i \in U}P(U)~\langle  x_i, \int \bar{v}(z)~\d P_U(z) \rangle .
\end{align*}
Hence defining $\bar{v}'(x)= P(U_x) \int \bar{v}(z)~\d P_{U_x}(z)  $ satisfies
\begin{align*}
 y_i\int \langle \bar{v}(z), x_i \rangle \mathbf{1}[ \langle x_i, z \rangle >0 ]~ \d \mu_{\mathcal{N}}(z)=
 y_i\int \langle \bar{v}'(z), x_i \rangle \mathbf{1}[ \langle x_i, z \rangle >0 ]~ \d \mu_{\mathcal{N}}(z)
\end{align*}
and since $\| \bar{v}(z)\|_2 \leq 1$ it follows that $\|\bar{v}'(z)\|_2 \leq 1$ for all $z\in \mathbb{R}^d$.
\end{proof}

Next we give an idea how the dimension $d$ can impact $\gamma$.
We show that in the simple case, where $\mathbb{R}^d$ can be divided into orthogonal subspaces, such that each data point $x_i$ is an element of one of the subspaces, there is a helpful connection between a mapping $\bar{v} \in \mathcal{F}_B$ and the mapping that $\bar{v}$ induces on the subspaces.

\begin{lem}\label{orthlem}
Assume there exist orthogonal subspaces $V_1, \dots V_s$ of $\mathbb{R}^d$ with $\mathbb{R}^d=\bigoplus_{j \leq s}V_j$ such that for each $i \in [n]$ there exists $j \in [s]$ such that $x_i \in V_j$.
Then the following two statements hold: 

{\bf Part 1.} Assume that for each $j \in [s]$ there exists $\gamma_j>0$ and $\bar{v}_j \in  \mathcal{F}_B$ such that for all $x_i \in V_j$ we have
\begin{align*}
y_i\int \langle \bar{v}_j(z), x_i \rangle \mathbf{1}[ \langle x_i, z \rangle >0 ]~ \d \mu_{\mathcal{N}}(z) \geq \gamma_j.
\end{align*}
Then for each $\rho\in \mathbb{R}^s$ with $\| \rho \|_2=1$ there exists $\bar{v} \in \mathcal{F}_B $ with
\begin{align*}
\min_{i \in [n] }y_i\int \langle \bar{v}(z), x_i \rangle \mathbf{1}[ \langle x_i, z \rangle >0 ]~ \d \mu_{\mathcal{N}}(z) \geq \min_{j \in [ s ] }\rho_j\gamma_j.
\end{align*}
\\

{\bf Part 2.} Assume that $\bar{v}$ maximizes the term
\[ \gamma^*=\min_{i\in [n]} y_i\int \langle \bar{v}(z), x_i \rangle \mathbf{1}[ \langle x_i, z \rangle >0 ]~ \d \mu_{\mathcal{N}}(z), \]
and that $\gamma^*>0$.
Given any vector $z\in \mathbb{R}^d$ we denote by $ p_j(z) \in V_j$ the projection of $z$ onto $V_j$.
Let $ \rho_j'= \max_{z \in \mathbb{R}^d} \| p_j(\bar v(z)) \|_2 $.
Then for all $j \in [s]$ the mapping $\bar{v}_j(z)=\frac{p_j(\bar v(z))}{\rho_j'}$ maximizes
\[ \gamma_j=\min_{x_i \in V_j} y_i\int \langle \bar{v}_j(z), x_i \rangle \mathbf{1}[ \langle x_i, z \rangle >0 ]~ \d \mu_{\mathcal{N}}(z) \]
and it holds that $  \| \bar{v}_j(z) \|_2 \leq 1$ for all $z \in \mathbb{R}^d$.
In other words if $\bar{v}$ is optimal for $(X, Y)$ then $\bar{v}_j$ is optimal for $(X_j, Y_j) $ where $X_j=\{ x_i\in V_j ~|~ i \in [n] \} $ with the corresponding labels, i.e., $y_{x_i}=y_i$.
\end{lem}

\begin{proof}

{\bf Part 1.} 

Since applying the projection $p_j$ onto $V_j$ to any point $z \in \mathbb{R}^d$ does not change the scalar product of $z$ and $x_i \in V_j$, i.e., $\langle x_i, z \rangle =\langle x_i, p_j(z) \rangle$, we can assume that for all $z \in \mathbb{R}^d$ we have $ \bar{v}_j(z) \in V_j$.
Let $z \in \mathbb{R}^d$.
We define $\bar{v}(z):=\sum_{j=1}^s \rho_j \bar{v}_j(z)$.
Then by orthogonality
\begin{align*}
\| \bar{v}(z) \|_2^2 = \sum_{j=1}^s \rho_j^2 \| \bar{v}_j(z) \|_2^2 \leq  \sum_{j=1}^s \rho_j^2 \cdot 1 = 1.
\end{align*} 
Thus it holds that $ \bar{v} \in \mathcal{F}_B$.
Further we have $\langle x_i, \bar{v}(z) \rangle=\sum_{k=1}^s  \rho_k\langle x_i, \bar{v}_k(z) \rangle = \rho_j\langle x_i, \bar{v}_j(z) \rangle$ for $x_i \in V_j$ again by orthogonality it holds that
\begin{align*}
y_i\int \langle \bar{v}(z), x_i \rangle \mathbf{1}[ \langle x_i, z \rangle >0 ]~ \d \mu_{\mathcal{N}}(z) 
= & ~ \rho_j y_i\int \langle \bar{v}_j(z), x_i \rangle \mathbf{1}[ \langle x_i, z \rangle >0 ]~ \d \mu_{\mathcal{N}}(z) \geq  ~ \rho_j \gamma_j .
\end{align*} 

{\bf Part 2.} 

For the sake of contradiction assume that there are $k \leq s$ and $\bar{v}^*_k \in \mathcal{F}_B$ such that
\[ \gamma^*_k=\min_{x_i \in V_k} y_i \int \langle \bar{v}_k(z), x_i \rangle \mathbf{1}[ \langle x_i, z \rangle >0 ]~ \d \mu_{\mathcal{N}}(z)= \gamma_k+ \varepsilon \]
for some $\varepsilon>0$.
Using {\bf Part 1.} we can construct a new mapping $\bar{v}' \in \mathcal{F}_B$ by using the mappings $\bar{v}_j$ defined in the lemma for $j \neq k$, and exchange $\bar{v}_k$ by $\bar{v}_k^*$. Also as in {\bf Part 1} let $\rho_j=\rho_j'+\varepsilon'$ for $j \neq k$ and $\rho_k=\rho_k'-2\frac{s\varepsilon'}{\rho_k'} $ with $\varepsilon'=\min\lbrace\frac{\rho_k'^2}{ 4s}, \frac{\rho_k'^2 \varepsilon}{4(\gamma_k+\varepsilon)s} \rbrace$.
Then we have
\begin{align*}
2s+s \varepsilon'+4s^2\frac{\varepsilon'}{\rho_k'^2}\leq 4s.
\end{align*}
Subtracting $4s$ and multiplying with $\varepsilon'$ gives us
\begin{align*}
2s\varepsilon'+s \varepsilon'^2-4s\varepsilon'+4\left( \frac{s\varepsilon'}{\rho_k'} \right)^2\leq 0.
\end{align*}
Hence it holds that
\begin{align*}
\sum_{j=1}^s \rho_j^2
\leq & ~ \left(\sum_{j\neq k} (\rho_j'^2+2\varepsilon'+\varepsilon'^2)\right)+\rho_k'^2-{4s\varepsilon'}+4\left( \frac{s\varepsilon'}{\rho_k'} \right)^2\\
\leq & ~ \left( \sum_{j=1}^s  \rho_j'^2 \right)+2s\varepsilon'+s \varepsilon'^2-4s\varepsilon'+4\left( \frac{s\varepsilon'}{\rho_k'} \right)^2 \leq  ~ \sum_{j=1}^s \rho_j'^2 \leq  ~ 1.
\end{align*}
For any $x_i \in V_j$ with $j\neq k$ we have by orthogonality as in {\bf Part 1.}
\begin{align*}
y_i\int \langle \bar{v}'(z), x_i \rangle \mathbf{1}[ \langle x_i, z \rangle >0 ]~ \d \mu_{\mathcal{N}}(z) &= \rho_j y_i\int \langle \bar{v}_j(z), x_i \rangle \mathbf{1}[ \langle x_i, z \rangle >0 ]~ \d \mu_{\mathcal{N}}(z) \\
&= (\rho_j' + \varepsilon') y_i \int \langle \bar{v}_j(z), x_i \rangle \mathbf{1}[ \langle x_i, z \rangle >0 ]~ \d \mu_{\mathcal{N}}(z) .
\end{align*} 
Further we have
\begin{align*}
\min_{x_i \in V_k} y_i \int \langle \bar{v}'(z), x_i \rangle \mathbf{1}[ \langle x_i, z \rangle >0 ]~ \d \mu_{\mathcal{N}}(z) 
= & ~ \rho_k \gamma^*_k \\ 
= & ~ (\rho_k'-2\frac{s}{\rho_k}\varepsilon')(\gamma_k+\varepsilon) \\
\geq & ~ \rho_k'\gamma_k-\frac{2s}{\rho_k'}\cdot \frac{\rho_k'^2 \varepsilon}{4(\gamma_k+\varepsilon)s}(\gamma_k+\varepsilon)+\rho_k'\varepsilon \\
= & ~ \rho_k'\gamma_k+\frac{\rho_k' \varepsilon}{2}.
\end{align*}
We conclude again by orthogonality that
\begin{align*}
y_i\int \langle \bar{v}'(z), x_i \rangle \mathbf{1}[ \langle x_i, z \rangle >0 ]~ \d \mu_{\mathcal{N}}(z) 
= & ~ \rho_j y_i \int \langle \bar{v}'_j(z), x_i \rangle \mathbf{1}[ \langle x_i, z \rangle >0 ]~ \d \mu_{\mathcal{N}}(z) \\
> & ~ \min_j \rho_j' \gamma_j \\
= & ~ \gamma^*
\end{align*}
and thus $\bar{v}'$ contradicts the maximizing choice of $\bar{v}$.
\end{proof}

As a direct consequence we get that the problem of finding an optimal $\bar{v}$ for the whole data set can be reduced to finding an optimal $\bar{v}_j$ for each subspace.

\begin{cor}\label{orthcor}
Assume there exist orthogonal subspaces $V_1, \dots V_s$ of $\mathbb{R}^d$ with $\mathbb{R}^d=\bigoplus_{j \leq s}V_j$ such that for each $i \in [n]$ there exists $j \in [s]$ with $x_i \in V_j$. For $j \in [s]$ let $(X_j, Y_j)$ denote the data set consisting of all data points $(x_i,y_i)$ where $x_i \in V_j$.
Then $\bar{v}$ is optimal for $(X, Y)$ if and only if for all $j\in [s]$ the mapping $\bar{v}_j$ defined in Lemma \ref{orthlem} is optimal for $(X_j, Y_j)$ and $\gamma_{\bar{v}}=\sum_{j \in [s]} \gamma_j \rho^*_j$ where $\rho^*=\argmax_{\rho \in \mathcal{S}^{s-1}}\min_{j\in [s]} \rho_j \gamma_j $.
\end{cor}

\begin{proof}
One direction follows immediately by Lemma \ref{orthlem} 2) the other direction is a direct consequence of the formula given in Lemma \ref{orthlem} 1).
\end{proof}

The following bound for $\gamma$ simplifies calculations in some cases of interest.
It also gives us a natural candidate for an optimal $ \bar{v}\in \mathcal{F}_B$. 
Given an instance $(X, Y)$ recall that $U_z=\{i\in[n] \mid \langle z,x_i\rangle > 0\}$. We set 
\begin{align}\label{eqv0}
    \bar{v}_0(z)=\frac{\sum_{i \in [n]\cap U_z}x_iy_i}{\| \sum_{i \in [n]\cap U_z}x_iy_i \|_2}.
\end{align}
We note that $\bar{v}_0(z) $ is not optimal in general but if instances have certain symmetry properties, then $\bar{v}_0(z) $ is optimal.

\begin{lem}\label{gammabound}
For any subset $S \subseteq [n]$ it holds that
\begin{align*}
\gamma \leq \sum_{U \subseteq [n]}P(U)\frac{1}{|S|}\left\| \sum_{i \in S\cap U}x_iy_i \right\|_2 
\end{align*}
\end{lem}

\begin{proof}
By Lemma \ref{psetlem} there exists an optimal $\bar{v}$ that is constant on $C(U)$ for all $U \subseteq [n]$.
For $x \in U$ let $z_U=\bar{v}(x) $.
Then by using an averaging argument and the Cauchy–Schwarz inequality we get
\begin{align*}
\gamma &\leq \frac{1}{|S|}\sum_{i \in S}y_i\int \langle \bar{v}(z), x_i \rangle \mathbf{1}[ \langle x_i, z \rangle >0 ]~ \d \mu_{\mathcal{N}}(z)\\
&= \frac{1}{|S|}\sum_{i \in S}y_i\sum_{U \subseteq [n], i\in U}P(U) \langle x_i, z_U \rangle\\
&= \frac{1}{|S|}\sum_{U \subseteq [n]}P(U) \langle \sum_{i \in S\cap U} y_i x_i, z_U \rangle\\
&\leq \sum_{U \subseteq [n]}P(U)\frac{1}{|S|}\left\| \sum_{i \in S\cap U}x_iy_i \right\|_2.
\end{align*}
\end{proof}

Finally we give an idea of how two points and their distance impacts the cones and their hitting probabilities.

\begin{lem}\label{twoplem}
Let $x_1, x_2 \in \mathcal{S}^{d-1}$ be two points with $\langle x_1, x_2 \rangle>0 $ and $ \| x_1 - x_2\|_2=b >0$.
Set $V_1'=\{ x \in \mathbb{R}^d ~|~ \langle x_1, x \rangle>0 \geq  \langle x_2, x \rangle \} $.
Then for a random Gaussian $z$ we have $z \in V_1'$ with probability $P(V_1')$ where $ \frac{b}{7} \leq P(V_1') \leq \frac{b}{5}$.
Further for any $z$ with $ \| z\|_2=1$ it holds that $| \langle x_1, z \rangle-\langle x_2, z \rangle |\leq b $.
\end{lem}

\begin{proof}
We define $V_1=\{ x \in \mathbb{R} ~|~ \langle x_1, x \rangle>0  \} $.
Then $P(V_1)=\frac{1}{2}$ since for a random Gaussian $z$ it holds that $\langle x_1, z \rangle>0$ with probability $\frac{1}{2}$.
Since the space spanned by $x_1$ and $x_2$ is $2$-dimensional, we can assume that $x_1$ and $x_2$ are on the unit circle and that $x_1=(1, 0)$ and $x_2=(\cos(\varphi), \sin(\varphi))$ for $\varphi \leq \frac{\pi}{2}$.
Note that $ P(V_1')$ is given by $ \frac{b'}{2\pi}$ where $b'=\varphi$ is the length of the arc connecting $x_1$ and $x_2$ on the circle.
Since $b$ is the Euclidean distance and thus the shortest distance between $x_1$ and $x_2$ we have $b\leq b'$.
Further it holds that
\[ h(\varphi):=\frac{b'}{b}=\frac{\varphi}{\sqrt{(1-\cos(\varphi))^2+\sin(\varphi)^2}}=
\frac{\varphi}{\sqrt{2-2\cos(\varphi)}}.\]
Then $h'(\varphi)$ is positive on $(0, \frac{\pi}{2}]$, so $ h(\varphi)$ is monotonously non-decreasing, and thus $ h(\varphi)\leq h(\frac{\pi}{2})=\frac{(\pi/2)}{\sqrt{2}}=\frac{\pi}{\sqrt{8}}$ and $b'\leq b \cdot \frac{\pi}{\sqrt{8}} $.
Consequently for $ P(V_1')=\frac{b'}{2\pi} $ we have that
\begin{align*}
\frac{b}{7}  \leq \frac{b}{2\pi} \leq P(V_1') \leq \frac{b}{2\pi}\cdot \frac{\pi}{\sqrt{8}} \leq \frac{b}{5}.
\end{align*}

For the second part we note that for any $z$ with $ \| z\|_2=1$ we get
\begin{align*}
| \langle z, x_1 \rangle - \langle z, x_2 \rangle |
=| \langle z, x_1- x_2 \rangle |
\leq \| z \|_2 \| x_1- x_2 \|_2 = 1 \cdot b
\end{align*}
by using the Cauchy–Schwarz inequality.
\end{proof}

\subsection{Example 1: Orthogonal unit vectors}
Let us start with a simple example first:
let $e_i\in \mathbb{R}^d$ be the $i$-th unit vector.
Let $n=2d$, $x_i=e_i$ for $i\leq d$ and $x_i=-e_{i-d}$ otherwise with arbitrary labels.
First consider the instance $(X_i, Y_i)$ created by the points $x_i$ and $x_{i+d}$ for $i\leq d$.
Then we note that $\bar{v}_i$ sending any point $z$ with $ \langle  z, e_i \rangle>0$ to $e_i y_i$ and any other point to $-e_i y_{i+d}$ is optimal since it holds that $\gamma_i=\gamma_{\bar{v}_i}(X_i, Y_i)= \int 1\cdot \mathbf{1}[ \langle x_i, z \rangle >0 ]~ \d \mu_{\mathcal{N}}(z) =\frac{1}{2} $.
Since the subspaces $V_i=\operatorname{span}\lbrace e_i \rbrace$ are orthogonal we can apply Corollary \ref{orthcor} with vector $\rho=(\frac{1}{\sqrt{d}})^d$.
Thus the optimal $\gamma$ for our instance is $\frac{1}{2 \sqrt{d}}$.

\subsection{Example 2: Two differently labeled points at distance \texorpdfstring{$b$}{b}}\label{ex:twopoints}

\begin{figure*}[ht!]
\begin{center}
\includegraphics[width=.5\linewidth]{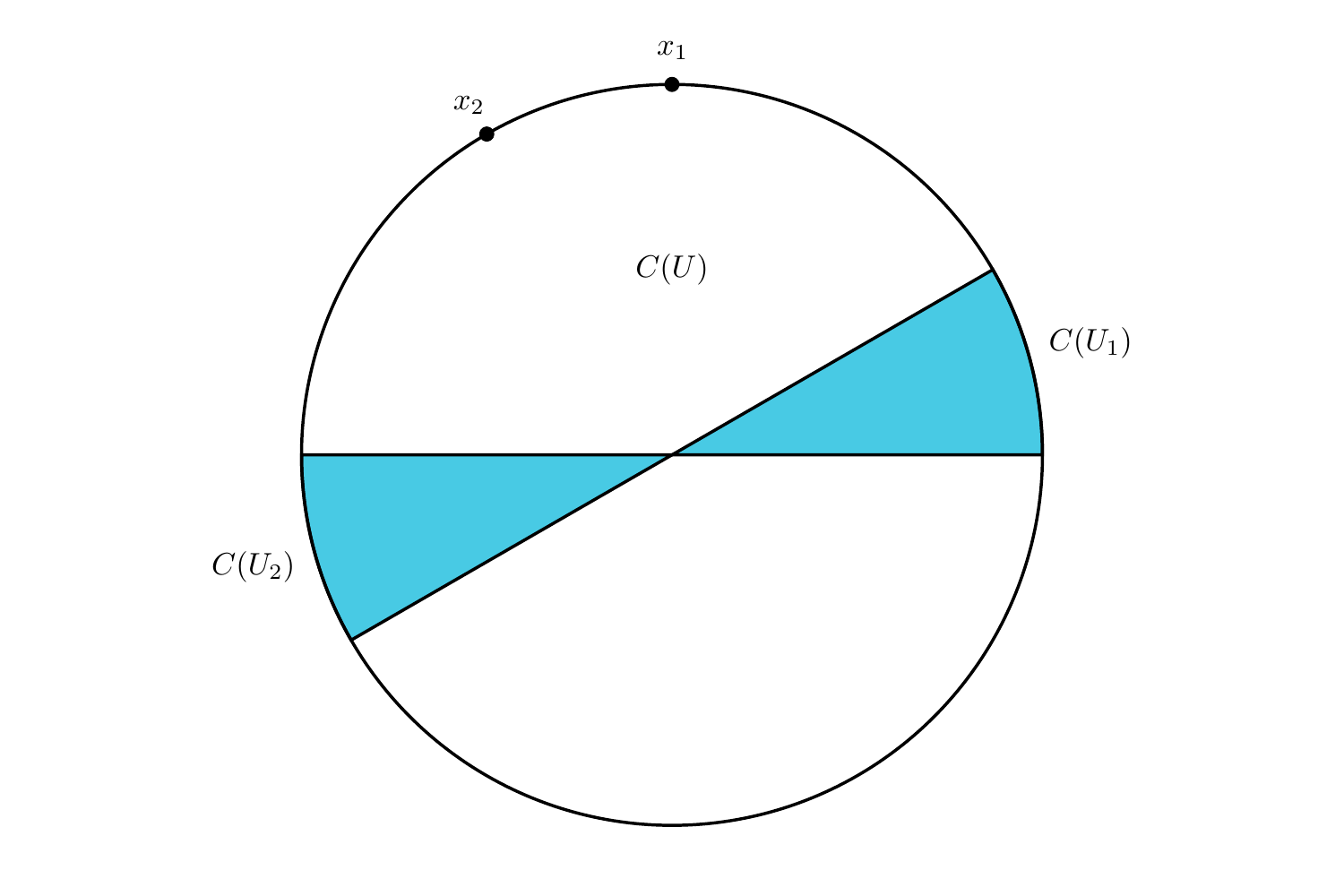}
\caption{a) Two points $x_1$ and $x_2$ on the sphere. $C(U)$ is the cone consisting of vectors having positive scalar product with both points. The cone $C(U_i)$ consists of vectors having positive scalar product with $x_i$ but negative scalar product with the other point. b) The probability $P(U_i)$ of a random Gaussian being in the cone $C(U_i)$ is exactly the length of the shortest arc on the circle (which is close to the Euclidean distance) connecting the points, divided by $2\pi$.
}
\label{fig:twopoints}
\end{center}
\end{figure*}

The next example is a set of two points $x_1, x_2 \in \mathbb{R}^d$ with $y_1=1=-y_2 $ and $\langle x_1, x_2 \rangle>0 $.
Let $U_1=\{1\}, U_2=\{2\}, U=\{ 1,  2 \}$ and $V_1=\{ x \in \mathbb{R} ~|~ \langle x_1, x \rangle>0  \} $.
Then $P(U)=P(V_1)- P(U_1)\geq \frac{1}{2}- \frac{b}{5}$ by Lemma \ref{twoplem} and $P(U_1)=P(U_2)=P(V_1)- P(U)\leq  \frac{1}{2}-(\frac{1}{2}- \frac{b}{5})=\frac{b}{5}$. For an illustration see Figure \ref{fig:twopoints}.

By Lemma \ref{psetlem} we can assume that there exists an optimal $\bar{v}$ which is constant on $C(U)$ and constant on $C(U_i)$ for $i \in \{1, 2\}$, i.e., that $\bar{v}(z)=z' \in B$ for all $z \in C(U)$ and $\bar{v}(z)=z'' \in B$ for all $z \in C(U_1)$.

By Lemma \ref{twoplem} we have $| \langle x_1, z' \rangle-\langle x_2, z' \rangle |\leq b $.
Consequently since $x_1$ and $x_2$ have different labels there exists at least one $i \in \{1, 2\}$ with $\langle z', x_i \rangle y_i \leq b/2$ since $\langle z', x_1 \rangle \geq b/2$ implies $-\langle z', x_2 \rangle \leq -\langle z', x_1 \rangle+  | \langle x_1, z' \rangle-\langle x_2, z' \rangle |\leq -b/2+b=b/2$ .
Then by Lemma \ref{psetlem} we have
\begin{align*}
y_i\int \langle \bar{v}(z), x_i \rangle \mathbf{1}[ \langle x_i, z \rangle >0 ]~ \d \mu_{\mathcal{N}}(z) 
\leq & ~ P(U)\cdot \langle z', x_i \rangle + P(U_i) \cdot \langle z'', x_i \rangle \\
\leq & ~ \frac{1}{2}\cdot \frac{b}{2}+\frac{b}{5} \cdot 1 \\
\leq & ~ \frac{b}{2}.\\
\end{align*}

\subsection{Example 3: Constant labels}

Let $X$ be any data set and let $Y$ be the all $1$s vector.
Then for $\bar{v}(z)=\frac{z}{\| z \|_2}$ it holds that
\begin{align*}
y_i\int \langle \bar{v}(z), x_i \rangle \mathbf{1}[ \langle x_i, z \rangle >0 ]~ \d \mu_{\mathcal{N}}(z) =
y_i\int  \left\langle \frac{z}{\| z \|_2}, x_i \right\rangle \mathbf{1}[ \langle x_i, z \rangle >0 ]~ \d \mu_{\mathcal{N}}(z)
\stackrel{*}{=}\Omega\left(\frac{1}{\sqrt{d}}\right).
\end{align*}
Thus we have $\gamma(X, Y) = \Omega\left(\frac{1}{\sqrt{d}}\right)$.
We note that $*$ is a well-known fact, see \citet{BlumHK20}. Since we consider only a fixed $x_i$, we can assume that $y_ix_i$ equals the first standard basis vector $e_1$. We are interested in the expected projection of a uniformly random unit vector $\frac{z}{\|z\|_2}$ in the same halfspace as $e_1$.

We give a short proof for completeness:
note that $\frac{z}{\|z\|_2}=(z_1, \dots, z_d)/\sqrt{\sum_{i=1}^d z_i^2}$ with  $z_i\sim \mathcal N(0,1)$, is a uniformly random unit vector $u$. By Jensen's inequality we have $\E[\sqrt{\sum_{i=1}^d z_i^2}] \leq \sqrt{ \E[\sum_{i=1}^d z_i^2] } = \sqrt{ \sum_{i=1}^d \E[ z_i^2] }=\sqrt{d}$. Thus, with probability at least $3/4$ it holds that $\sqrt{ \sum_{i=1}^d g_i^2 } \leq 4 \sqrt{d}$, by a Markov bound. Also, $|z_i| \geq \sqrt{2}\cdot \erf^{-1}(1/2)$ holds with probability at least $1/2$, since the right hand side is the median of the half-normal distribution, i.e., the distribution of $|z_i|$, where $z_i\sim\mathcal{N}(0,1)$. Here $\erf$ denotes the the Gauss error function.

By a union bound over the two events it follows with probability at least $1-\frac 1 2 -\frac 1 4 =\frac 1 4 $ that
\begin{align*}
|u_i| = |z_i|/\sqrt{\sum_{i=1}^d z_i^2} \geq \sqrt{2}\cdot \erf^{-1}(1/2)/(4 \sqrt{d}).
\end{align*}

Consequently $\E[|u_i|]\geq \frac 1 4 \cdot \sqrt{2}\cdot \erf^{-1}(1/2)/(4 \sqrt{d})= \Omega(1/\sqrt{d})$ and thus
\begin{align*}
 y_i\int  \left\langle \frac{z}{\| z \|_2}, x_i \right\rangle \mathbf{1}[ \langle x_i, z \rangle >0 ]~ \d \mu_{\mathcal{N}}(z)=\frac{1}{2}\E[|u_i|] = \Omega(1/\sqrt{d}).
\end{align*}

\subsection{Example 4: The hypercube}

In the following example we use $x_i$ for the $i$-th coordinate of $x\in \mathbb{R}^d$ rather than for the $i$-th data point.
We consider the hypercube $X=\{- \frac{1}{\sqrt{d}}, + \frac{1}{\sqrt{d}}\}^d$ with different labelings.
Given $x \in X$ we set $S_x=\{ i \in [d] ~|~ x_i=- \frac{1}{\sqrt{d}} \}$ and $\sigma(x)=| S_i|$.

\subsubsection{Majority labels}

First we consider the data set $X'=X \setminus \{ x \in X ~|~ \sigma(x)=\frac{d}{2} \}$ and assign $y_x=-1$ if $\sigma(x)>\frac{d}{2}$ and $y_x=-1$ if $\sigma(x)<\frac{d}{2}$.
Note that $ d-2\sigma(x)<0$ holds if and only if $y_x=-1$.
Let $x_c \in X$ be the constant vector that has all coordinates equal to $1/\sqrt{d}$. Now, if we fix $\bar{v}(z)=x_c$ for any $z$, then for all $x \in X'$ we have that
\[  y_x \int \langle \bar{v}(z), x \rangle \mathbf{1}[ \langle x, z \rangle >0 ]~ \d \mu_{\mathcal{N}}(z)=\frac{y_x}{2}\cdot \frac{d-2\sigma(x)}{d}\geq \frac{1}{2}\cdot \frac{1}{d} . \]
Hence it follows that $\gamma(X', Y)\geq \frac{1}{2d}$

\subsubsection{Parity labels}

Second we consider the case where $y_x=(-1)^{\sigma(x)}$.
Then we get the following bounds for $\gamma$:

\begin{lem}\label{lemhc}
Consider the hypercube with parity labels.
\begin{itemize}
\item[1)]If $d$ is odd, then $\gamma=0$.
\item[2)]If $d$ is even, then $\gamma>0$.
\end{itemize}
\end{lem}

\begin{proof}
1): First note that the set $Z=\{ z \in \mathbb{R}^d ~|~ \exists x \in X \text{ with } \langle x , z \rangle =0  \}$ is a null set with respect to the Gaussian measure $\mu_{\mathcal{N}}$.
Fix any coordinate $i \leq d$.
W.l.o.g. let $i\neq 1$.
Given $ x\in M:= \{ \frac{1}{\sqrt{d}} \} \times \{ -\frac{1}{\sqrt{d}}, \frac{1}{\sqrt{d}} \}^{d-2}$ consider the set $S(x)=\{(\frac{1}{\sqrt{d}},x), (-\frac{1}{\sqrt{d}}, x), (\frac{1}{\sqrt{d}}, -x), (-\frac{1}{\sqrt{d}}, -x)\}$.
Note that $X$ is the disjoint union $X=\dot{\bigcup}_{x\in M }S(x) $.
Further since $d-1$ is even, it holds that $y_{(\frac{1}{\sqrt{d}},x)}=y_{(\frac{1}{\sqrt{d}},-x)}= -y_{(-\frac{1}{\sqrt{d}},-x)} = -y_{(-\frac{1}{\sqrt{d}},x)}$.
Let $z\in Z$ and let $U_z=\{ x'\in X ~|~ \langle z, x' \rangle>0 \}$.
W.l.o.g. let $ \langle z, (\frac{1}{\sqrt{d}},x) \rangle>0$.
Then we have $ \langle z, x' \rangle>0$ for exactly one $x' \in \{ (-\frac{1}{\sqrt{d}}, x), (\frac{1}{\sqrt{d}}, -x) \}$ and $\langle z, (-\frac{1}{\sqrt{d}},-x) \rangle<0$.
Now since $y_{(\frac{1}{\sqrt{d}},x)}(\frac{1}{\sqrt{d}},x)_i= -y_{(\frac{1}{\sqrt{d}},x)}(-1,x)_i=-y_{(\frac{1}{\sqrt{d}},x)}(1,-x)_i$ we conclude that for all $x \in M$ it holds that 
\begin{align*}
\sum_{x' \in S(x)\cap U_z}(x'y_{x'})_i=\frac{1}{\sqrt{d}}+ (- \frac{1}{\sqrt{d}})=0
\end{align*}
and thus we get
\begin{align*}
\sum_{x \in X\cap U_z}(xy_x)_i=\sum_{x\in M}\sum_{x' \in S(x)\cap U_z}(xy_x)_i=0.
\end{align*}
Thus by Corollary \ref{gammabound} it holds that $\gamma=0$.

2): Consider the set $ M$ comprising the middle points of the edges, i.e., $ M=\{ x \in \{-\frac{1}{\sqrt{d}}, 0, \frac{1}{\sqrt{d}}\}^d ~|~ x_i=0 \text{ for exactly one coordinate $i \in [d]$} \} $.
Observe that for any $x \in X$ and $z \in M$ the dot product $d\cdot \langle x, z \rangle$ is an odd integer and thus $ |\langle x, z \rangle|\geq 1/d$.
Hence, for the cone $ C(U_z)$ containing $z$ we have $P(U_z)>0$.

Now fix $z \in M$ and let $i \in [d]$ be the coordinate with $z_i=0$.
Recall $\sigma(z)=|\{ k \in [d]~|~ z_k=-\frac{1}{\sqrt{d}} \}|$ and set $\bar{v}(z)=e_i\cdot\sigma(z) \cdot (-1)^{d/2+1}$.
Let $j \in [d]$ be any coordinate other than $i$ and consider the pairs $ \{ v, w \} \subset X$ where $v \in X$ with $v_j=z_j$, $\langle v, z \rangle>1/d$ and $w=v-2v_je_j$.
We denote the union of all those pairs by $V'$.
The points $v$ and $w$ have the same entry at coordinate $i$ but different labels.
Hence it holds that $\sum_{(v, w) \in V'}v_i y_v+w_i y_w=0 $.

Next consider the set of remaining vectors with $\langle v, z \rangle>0$ which is given by $V=\{ x \in X ~|~ x_j=z_j \text{ and } \langle x, z \rangle=1/d \}$.
For all $x \in V $ with $x_i=\frac{1}{\sqrt{d}}$ it holds that $\sigma(x)=\sigma(z)-(\frac{d}{2}-1)=\sigma(z) \cdot (-1)^{d/2+1}$ since the projection of $x$ to $\mathbb{R}^{d-1}$ that results from removing the $i$-th entry of $x$, has Hamming distance $(\frac{d}{2}-1)$ to $z$ projected to $\mathbb{R}^{d-1}$, 
and vice versa for all $x \in V $ with $x_i=-1/\sqrt{d}$ we have that $\sigma(x)=\sigma(z) \cdot (-1)^{d/2}$.
Hence for $x \in V$ it holds that $y_x\bar{v}(z) = e_i \cdot \sigma(z) \cdot (-1)^{d/2+1} = e_i \cdot \sgn(x_i)$ and thus we have
\begin{align*}
\sum_{x \in X\cap U_z} y_x\langle x, \bar{v}({z}) \rangle
&=\sum_{x \in V}y_x\langle x, \bar{v}({z}) \rangle ~+~ \sum_{(v, w) \in V'}y_v\langle v, \bar{v}({z}) \rangle+y_w\langle w, \bar{v}({z}) \rangle\\
&=\sum_{x \in V}\sgn(x_i)\langle x,  e_i \rangle ~+~ 0\\
&=\sum_{x \in V}\frac{1}{\sqrt{d}}
=2\binom{d-1}{d/2-1}\frac{1}{\sqrt{d}}
\end{align*}
since the number of elements $x \in V$ with $x_i=1/\sqrt{d} $ is the same as the number of elements $x' \in V$ with $x_i'=-1/\sqrt{d} $.
More specifically, it equals the number of points with Hamming distance $(\frac{d}{2}-1)$ to the projection of $z$ onto $\mathbb{R}^{d-1}$, which is $ \binom{d-1}{d/2-1}$ since the $i$-th coordinate is fixed and we need to choose $d/2-1$ coordinates that differ from the remaining coordinates of $z$.
Let $P>0$ be the probability that a random Gaussian is in the same cone $C(U)$ as $z$ for some $z \in M$.
Then by symmetry it holds that $\gamma_{\bar{v}}=P\cdot 2\binom{d-1}{d/2-1}\cdot \frac{1}{\sqrt{d}} \cdot\frac{1}{|X|}>0$.
\end{proof}

\section{Lower bounds for log width}\label{sec:LB_logwidth}

\subsection{Example 5: Alternating points on a circle}\label{ex:cycle}
Next consider the following set of $n$ points for $n$ divisible by $4$:\\
$x_k=\left(\cos\left(\frac{2k \pi}{n}\right),  \sin\left(\frac{2k \pi}{n}\right) \right)$ and $y_k=(-1)^k$.
Intuitively, defining $\bar{v}$ to send $z \in \mathbb{R}^d$ to the closest point of our data set $X$ multiplied by its label, gives us a natural candidate for $\bar{v}$.
However, applying Lemma \ref{psetlem} gives us a better mapping that also follows from Equation (\ref{eqv0}), and which is optimal by Lemma \ref{gammabound}:\\
Define the set $S=\lbrace x \in \mathbb{R}^2 \mid \exists x_i\in X, \alpha\geq 0\colon x=\alpha x_i \rbrace$ 
Now, for any $z \in \mathbb{R}^d \setminus S$ there exists a unique $i_z$ such that $z\in \Cone (\{x_{i_z}, x_{i_z+1}\})$.
We set $r_z=\frac{x_{i_z}y_{i_z}+x_{i_z+1}y_{i_z+1}}{\| x_{i_z}-x_{i_z+1} \|_2}$.
We define the function $\bar{v} : \mathbb{R}^d \rightarrow \mathbb{R}^d$ by
\[
\bar{v}(z)=
\begin{cases}
0 & \text{$z\in S$} \\
(-1)^{n/4+1}r_z  & \text{ otherwise.}
\end{cases}
\]
Observe that for $i=i_z$ we have 
\begin{align*}
r_z&=\left(\cos\left(\frac{2 \pi}{2n}\cdot (i-\frac{n}{2}+1)\right),  \sin\left(\frac{2 \pi}{2n}\cdot (i-\frac{n}{2}+1)\right) \right) \\
&=(-1)^i\left(\sin\left(\frac{(i+1)2 \pi}{2n}\right),  -\cos\left(\frac{(i+1)2 \pi}{2n} \right) \right).
\end{align*}
Figure \ref{fig:altcircle} shows how $ \bar{v}(z)$ is constructed for $ n=12$.
We note that $\bar{v}=\bar{v}_0 $ holds almost surely, which in particular implies the optimality of $\bar{v}$, cf. Equation (\ref{eqv0}).
For computing $\gamma$ we need the following lemma.  

\begin{figure*}[ht!]
\begin{center}
\begin{tabular}{cccc}
\includegraphics[width=0.4\linewidth]{figure_1.pdf}\label{fig:altcircle1}&
\includegraphics[width=0.4\linewidth]{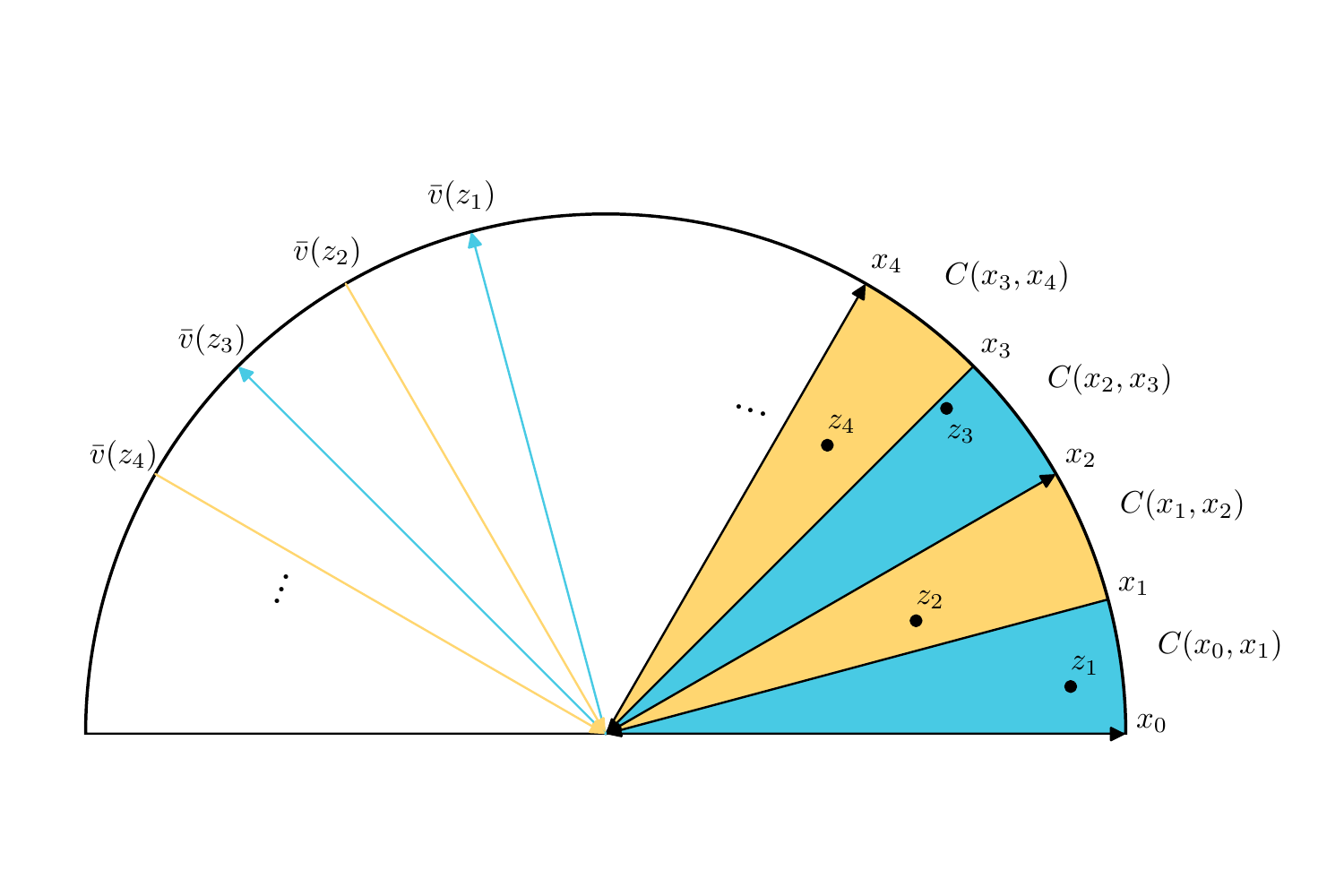}\label{fig:altcircle2}&
\end{tabular}
\caption{The left picture shows how $\bar{v}(z)$ is constructed: we subtract the vector $x_3$ which is labeled $-1$ from the vector $x_2$ which is labeled $1$. We obtain $r_z$ after rescaling to unit norm. Since $n/4=3$ is odd we have $\bar{v}(z)=r_z$. The right picture demonstrates the values of $\bar{v}(z)$ that are relevant for computing $y_i\int \langle \bar{v}(z), x_i \rangle \mathbf{1}[ \langle x_i, z \rangle >0 ] \d \mu_{\mathcal{N}}(z)$ for the single point $x_i=(0, 1)$. Here we have $ y_i\langle x_i, \bar{v}(z_j) \rangle=(-1)^{j-1}\cos\left(\frac{(2j-1) \pi}{2n}\right)$. The same argument can be repeated on the left side of the half circle.}
\label{fig:altcircle}
\end{center}
\end{figure*}

We found the result in a post on \texttt{math.stackexchange.com} but could not find it in published literature and so we reproduce the full proof from \cite{stex1} for completeness of presentation.

\begin{lem}[\cite{stex1}]\label{lemaltcyclhelp}
For any $a, b \in \mathbb{R}$ and $\tilde{n}\in \mathbb{N} $ it holds that
\[ \sum_{k=0}^{\tilde{n}-1}\cos(a+k b)=\frac{\cos(a+(\tilde{n}-1)b/2)\sin(\tilde{n}b/2)}{\sin(b/2)}.\]
\end{lem}

\begin{proof}
We use $\i$ to denote the imaginary unit defined by the property $\i^2 = -1$. 
From Euler's identity we know that $\cos(a+kb)=\text{Re}(e^{\i (a+kb)})$ and $\sin(a+kb)=\text{Im}(e^{\i (a+kb)})$. Then
\begin{align*}
\sum_{k=0}^{\tilde{n}-1} \cos (a+kb) &= \sum_{k=0}^{\tilde{n}-1} \text{Re}\left(e^{\i (a+kb)}\right)\\
&=\text{Re}\left(\sum_{k=0}^{\tilde{n}-1} e^{\i (a+kb)}\right)\\
&=\text{Re}\left(e^{\i a} \sum_{k=0}^{\tilde{n}-1} (e^{\i b})^{k} \right)\\
&=\text{Re} \left( e^{\i a} \frac{1-e^{\i b\tilde{n}}}{1-e^{\i b}}\right) \\
&=\text{Re} \left( e^{\i a} \frac{e^{\i b\tilde{n}/2}(e^{-\i b\tilde{n}/2}-e^{\i b\tilde{n}/2})}{e^{\i b/2}(e^{-\i b/2}-e^{\i b/2})}\right) \\
&=\frac{\cos(a+(\tilde{n}-1)b/2)\sin(\tilde{n}b/2)}{\sin(b/2)}.
\end{align*}
\end{proof}

\begin{lem}\label{Prop5.3}
For all $i\in [n]$ it holds that
\[ y_i\int \langle \bar{v}(z), x_i \rangle \mathbf{1}[ \langle x_i, z \rangle >0 ] \d \mu_{\mathcal{N}}(z)=\Omega\left(\frac{1}{n}\right).\]
\end{lem}

\begin{proof}
We set $n'=n/4$.
Note that by symmetry the value of the given integral is the same for all $i \in [n]$.
Thus it suffices to compute $ y_i\int \langle \bar{v}(z), x_i \rangle \mathbf{1}[ \langle x_i, z \rangle >0 ] \d \mu_{\mathcal{N}}(z)=\gamma$ for $x_i=(0, 1)$, and note that $i=n/4$ for this special choice.
See Figure \ref{fig:altcircle} for an illustration of the following argument.
For a fixed $z \in \mathbb{R}^2$ consider the cone $\Cone (\{x_{i_z}, x_{i_z+1}\}=\Cone (\{x_{j}, x_{j+1}\})\subset \{ x\in \mathbb{R}^2 ~|~ \langle x, (0, 1) \rangle >0\}$.
Then $j \in [0, \frac{n}{2}-1]$ and $\langle \bar{v}(z), x_i \rangle= (-1)^{n/4+1}\langle r_z, x_i \rangle=(-1)^{n/4+1}(-1)^{j+1}\cos(\frac{(2j+1)2 \pi}{2n} )$ since $y_i= (-1)^{n/4}$.
Further, for $j \leq \frac{n}{4}-1$ it holds that
\begin{align*}
\langle \bar{v}(z), x_i \rangle=(-1)^{n/4}(-1)^{j}\cos\left(\frac{(2j+1)2 \pi}{2n} \right)&=y_i(-1)^j\cos\left(\frac{(2j+1)2 \pi}{2n} \right),
\end{align*}
and by using the symmetry of $\cos$ we get
\begin{align*}
(-1)^{n/4}(-1)^{(n/2)-j-1}\cos\left(\frac{(2(n/2)-2j-1)2 \pi}{2n} \right)&=y_i(-1)^{j+1}\left(-\cos\left(\frac{(2j+1)2 \pi}{2n} \right)\right)\\
&=y_i(-1)^j\cos\left(\frac{(2j+1)2 \pi}{2n} \right).
\end{align*}
Now assume w.l.o.g. that $n \geq 8$.
Further we set $\tilde{n}= (n'-1)/2$ and $ b=\frac{4\pi}{n}=\frac{4\pi}{4n'}$.
By using Lemma \ref{lemaltcyclhelp} and the Taylor series expansion of $\cos(\cdot)$ and $\sin(\cdot)$ we get
\begin{align*}
\gamma&=\frac{1}{n}\left(2\sum_{k=1}^{n'}\cos\left(\frac{(2k-1) \pi}{4n'}\right)(-1)^{k-1}\right) \\
&=\frac{2}{n}\left(\sum_{k=0}^{\lceil(n'-1)/2\rceil }\cos\left(\frac{(4k+1) \pi}{4n'}\right) - \sum_{k=0}^{\lfloor(n'-1)/2\rfloor}\cos\left(\frac{(4k+3) \pi}{4n'}\right) 
\right)\\
&\stackrel{*}{\geq}\frac{2}{n}\left(\frac{\cos(\pi/n+(\tilde{n}-1)b/2)\sin(\tilde{n}b/2)}{\sin(b/2)}-\frac{\cos(3\pi/n+(\tilde{n}-1)b/2)\sin(\tilde{n}b/2)}{\sin(b/2)} \right)\\
&=\frac{2}{n}\left(\frac{\overbrace{(\cos(\pi/n+(\tilde{n}-1)b/2)-\cos(3\pi/n+(\tilde{n}-1)b/2)}^{=\Theta(b)})\overbrace{\sin(\tilde{n}b/2)}^{=1-\Theta(b)}}{\sin(b/2)}
\right)\\
&=\frac{2}{n}\left(\frac{\Theta(b)}{\Theta(b)}\right)
= \frac{2}{n}\Theta(1)
=\Omega(n^{-1}).
\end{align*}
$*$ when $n'$ is odd then we have an exact equality.
\end{proof}

\subsection{Lower Bounds}

\begin{lem}\label{Prop5.4}
If $m= o(n \log(n))$ then with constant probability over the random initialization of $W_0$ it holds for any weights $V \in \mathbb{R}^{m \times d}$ that $y_i \langle V, \nabla f_i(W_0) \rangle \leq 0$ for at least one $i \in [n]$.
\end{lem}

\begin{proof}
We set $x_{-i}:=x_{n-i}$ for $i \geq 0$.
Consider the set $\{ x_i, x_{i+1}, x_{i+2}, x_{i+3}  \}$ for $i$ with $i \mod 4=0$.
For any $s$ let $A_{i, s}$ denote the event that
\[ \mathbf{1}[ \langle x_i, w_s \rangle >0 ]=\mathbf{1}[ \langle x_{i+1}, w_s \rangle >0 ]
=\mathbf{1}[ \langle x_{i+2}, w_s \rangle >0 ]=\mathbf{1}[ \langle x_{i+3}, w_s \rangle >0 ]. \]
If there exists $i \in \{0, 4, \dots , n-4 \}$ such that for all $s \in [m]$ the event $A_{i, s}$ is true then at least one of the points $x_i, x_{i+1}, x_{i+2}, x_{i+3}$ is misclassified.
To see this, note that there exists $\rho \in \mathbb{R}_{> 0}^4$ such that $\rho_1x_i+\rho_3 x_{i+2}-(\rho_2 x_{i+1} + \rho_4 x_{i+3}) =0$ since the line connecting $x_i $ and $x_{i+3}$ crosses the line segment between $x_{i+2}$ and $x_{i+4}$.
Now let $S=\{ s\in [m] ~|~ \langle x_i, w_s \rangle >0 \}$.
If the event $A_{i, s}$ is true for all $s \in [m]$ then it holds that
\begin{align*}
0&=\sum_{s\in [m], \langle x_i, w_s \rangle >0} \langle \rho_1x_i+\rho_3 x_{i+2}-(\rho_2 x_{i+1} + \rho_4 x_{i+3}), w_s \rangle \\
&= \sum_{j=0}^3 \sum_{s\in [m], \langle x_i, w_s \rangle >0} \rho_j y_{i+j}\langle  x_{i+j}, w_s \rangle \\
&= \sum_{j=0}^3 \rho_j \sum_{s\in [m], \langle x_{i+j}, w_s \rangle >0} y_{i+j}\langle  x_{i+j}, w_s \rangle 
\end{align*}
and since $\rho_j > 0$ it must hold $ \sum_{s\in [m], \langle x_{i+j}, w_s \rangle >0}  y_{i+j}\langle  x_{i+j}, w_s \rangle \leq 0$ for at least one $j\in \{0,\ldots, 3\}$.

Note that $A_{i, s}$ is false with probability $2 \cdot \frac{3}{n}$, namely if $\frac{w_s}{\| w_s\|_2}$ is between the point $x_{i+n/4}$ and $x_{i+3+n/4}$ or between the points $x_{i-n/4}$ and $x_{i+3-n/4}$.
We denote the union of these areas by $Z_i$.
Further these areas are disjoint for different $i, i' \in \{0,  4,\dots n/4\}$.
Now, as we have discussed above, we need at least one $A_{i,s}$ to be false for each $i$. This occurs only if for each $i$ there exists at least one $s$ such that $ \frac{w_s}{\| w_s\|_2} \in Z_i$.
Let $T$ be the minimum number of trials needed to hit every one of the $n':=n/4$ regions $Z_i$. This is the coupon collector's problem for which it is known \cite{er61} that for arbitrary $c\in \mathbb{R}$ it holds that $\Pr[T < n' \log n' + cn']=\exp(-\exp(-c))$ as $n'\rightarrow \infty$. Thus for sufficiently large $n'$ and $c=-1$ we have 
\begin{align*}
\Pr [ T > n' \log n' - n' ] > 1 - e^{-e} > 0.9.
\end{align*}
\end{proof}

Indeed we can show an even stronger result:
\begin{lem}\label{Prop5.4'}
Let $\epsilon \geq 0 $. Any two-layer ReLU neural network with width $m < (1-\varepsilon)n/6-2 $ misclassifies more than $\varepsilon n / 3$ points of the alternating points on the circle example.
\end{lem}

\begin{proof}
Set $\mathcal{D}=\lbrace x \in \mathbb{R}^2 ~|~ \| x \|_1=1 \rbrace$.
Given parameters $W$ and $a$ consider the function $f: \mathbb{R}^2 \rightarrow \mathbb{R}$ given by $f(x)=\frac{1}{\sqrt{m}}\sum_{s=1}^m a_s \phi\left(\langle w_s, x \rangle\right)$.
Note that the points $x_i'=\frac{x_i}{\| x_i \|_1} \in \mathcal{D}$ do not change their order along the $\ell_1$ sphere and thus by definition of $(x_i,y_i)$ have alternating labels. Also note that $f(x_i)>0$ if and only if $f(x_i')>0$.
Further note that the restriction of $f$ to $\mathcal D$ denoted $f_{|\mathcal{D}}$ is a piecewise linear function.
More precisely the gradient $\frac{\partial f}{\partial x}=\frac{1}{\sqrt{m}}\sum_{s=1}^m a_s \mathbf{1}[\langle w_s, x \rangle > 0]w_s$ can only change at the points $(1,0), (0, 1), (-1,0), (0, -1)$ and at points orthogonal to some $w_s$ for $s \leq m$.
Since for each $w_s$ there are exactly two points on $\mathcal{D}$ that are orthogonal to $w_s$ this means the gradient changes at most $2m+4 $ times.
Now for $i$ divisible by 3 consider the points $x_i, x_{i+1}, x_{i+2}$.
If the gradient does not change in the interval induced by $x_i$ and $x_{i+2}$ then at least one of the three points is misclassified.
Hence if $2m+4< (1-\varepsilon)\frac{n}{3}$ then \emph{strictly} more than an $(\varepsilon/3)$-fraction of the $n$ points is misclassified.
\end{proof}

\input{arxivlogwidth_upper}

\section{On the construction of \texorpdfstring{$U$}{U}}\label{sec:logwidth:end}

\subsection{Tightness of the construction of \texorpdfstring{$U$}{U}}

We note that for the construction of $U$ used in the upper bound of Lemma \ref{lem2.3} $m' \geq \frac{8\ln(2n/\delta)}{\gamma^2}$ is tight in the following sense: For $\bar{v} \in \mathcal{F}_B$, the natural estimator of $\gamma$ is given by the empirical mean $ \frac{1}{m}\sum_{s=1}^{m'}y_i\langle \bar{v}(w_{s, 0}), x_i \rangle]\mathbf{1}\left[ \langle x_i, w_{s, 0} \rangle >0 \right]$.
The following lemma shows that using this estimator, the bound given in Lemma \ref{lem2.3} is tight with respect to the squared dependence on $\gamma$ up to a constant factor.
In particular we need $m =\Omega(\gamma^{-2}\log(n))$ if we want to use the union bound over all data points.

\begin{lem}\label{lemweakupbo}
Fix the choice of $u_s= \frac{a_s}{\sqrt{m}} \bar{v}(w_s)$ for $s \in [m]$.
Then for each $\gamma_0\in (0,1)$ there exists an instance $(X, Y)$ and $ \bar{v}(z) \in \mathcal{F}_B$,
such that for each  $i \in [n]$ it holds with probability at least $P_m=c\exp\left(-8m'\gamma^2/3 \right)$ for an absolute constant $c>0$ that
\begin{align*}
y_i  f^{(0)}_i(U)=\frac{1}{m'}\sum_{s=1}^{m'}y_i\langle \bar{v}(w_{s, 0}), x_i \rangle]\mathbf{1}\left[ \langle x_i, w_{s, 0} \rangle >0 \right]\leq 0.
\end{align*}
\end{lem}
\begin{proof}[Proof of Lemma \ref{lemweakupbo}]
Consider Example \ref{ex:cycle}. Recall that $\gamma(X, Y)=\Theta(1/n)$. Choose a sufficiently large $n$, divisible by $8$, such that $ \gamma(X, Y) \leq \gamma_0$.
Note that the mapping $\bar{v}$ that we constructed, has a high variance since for any $i$, the probability that a random Gaussian $z$ satisfies $ \langle \bar{v}(z), x_i \rangle \mathbf{1}[ \langle x_i, z \rangle >0 ]\geq \frac{1}{\sqrt{2}}$ as well as the probability that $ \langle \bar{v}(z), x_i \rangle \mathbf{1}[ \langle x_i, z \rangle >0 ]\leq -\frac{1}{\sqrt{2}}$ are equal to $\frac{1}{8}$.
To see this, note that $|\langle\bar{v}(z), x_i \rangle| \geq \frac{1}{\sqrt{2}}$ if $\langle z, x_i \rangle < \frac{1}{\sqrt{2}}$ and in this case $\langle\bar{v}(z), x_i \rangle$ is negative with probability $\frac{1}{2}$.
Thus the variance of $Z_s=y_i\langle \bar{v}(w_{s, 0}), x_i \rangle]\mathbf{1}\left[ \langle x_i, w_{s, 0} \rangle >0 \right]$ is at least $ \frac{1}{\sqrt{2}^2}\cdot \frac{2}{8}=\frac{1}{8}$.
Observe that the random variable $Z_s'=\frac{1}{2}(1-Z_s)$ attains values in $ [0,1]$.
Further the expected value of $Z_s'$ is $\frac{1}{2}(1-\gamma)$, and the variance is at least $\frac{1}{32}$.
Now set $Z=\sum_{s=1}^{m'} Z'_s$ and note that $y_i f^{(0)}_i(U)=\frac{1}{m'}\sum_{s=1}^{m'}y_i\langle \bar{v}(w_{s, 0}), x_i \rangle]\mathbf{1}\left[ \langle x_i, w_{s, 0} \rangle >0 \right]\leq 0$ holds if and only if $Z \geq \frac{m'}{2}=\mathbb{E}(Z)+\frac{m' \gamma}{2} $.
By Lemma \ref{Prop5.4'} we know that $y_i f^{(0)}_i(U)=\frac{1}{m'}\sum_{s=1}^{m'}y_i\langle \bar{v}(w_{s, 0}), x_i \rangle]\mathbf{1}\left[ \langle x_i, w_{s, 0} \rangle >0 \right]\leq 0$ is true for at least one $i\in [n]$ if $m\leq \frac{n}{6}-3$.
Now choosing $n$ large enough this implies we only need to show the result for $m' \geq {200^2}\cdot{32}$.
Hence we can apply Lemma \ref{lemweakupbohelp} to $Z$ and get
\begin{align*}
\Pr[ Z\geq \mathbb{E}(Z)+\frac{m'\gamma}{2} ] \geq c\exp\left(-m'^2\gamma^2/ \left(\frac{4\cdot 3m'}{32}\right)\right)
=c\exp\left(-8m'\gamma^2/3 \right)
\end{align*}
for $ \frac{m'\gamma}{2} \leq \frac{1}{100}\frac{m'}{32}$ or equivalently $\gamma\leq \frac{1}{1600}$ which holds if $n$ is large enough.
\end{proof}

Thus we need that $m=\Omega(\frac{\ln(n/\delta)}{\gamma^2})$ for the given error probability if we construct $U$ as in Lemma \ref{lem2.3}.

\subsection{The two dimensional case (upper bound)}

In the following we show how we can improve the construction of $U$ in the special case of $d=2$ such that 
\begin{align*}
m= O \left(  \gamma^{-1} \left({\ln(4n/\delta)} +  \ln(4/\varepsilon) \right) \right)
\end{align*}
suffices for getting the same result as in Theorem \ref{mainthm}.
We note that the only place where we have a dependence on $\gamma^{-2}$ is in Lemma \ref{lem2.3}. It thus suffices to replace it by the following lemma that improves the dependence to $\gamma^{-1}$ in the special case of $d=2$:

\begin{lem}\label{d2lem2.3}
Let $(X, Y)$ be an instance in $d=2$ dimensions.
Then there exists a constant $K>1$ such that for $m\geq \frac{K{\ln(n/\delta)}}{\gamma}$ with probability $1- 2\delta$ there exists $U \in \mathbb{R}^{m \times d}$ with $\|u_s\|_2 \leq \frac{1}{\sqrt{m}}$ for all $s\leq m$, and $\|  U \|_F\leq 1$, such that 
\begin{align*}
y_i f^{(0)}_i(U)\geq \frac{\gamma}{4}
\end{align*}
for all $i\leq n$.
\end{lem}

\begin{proof}
The proof consists of three steps. 
The first step is to construct a net $X'$ that consists only of `large cones of positive volume' such that for each data point $x$ there exists a point $x' \in X'$ whose distance from $x$ on the circle is at most $b=\frac{\gamma}{4} $:
Let $n'=\lceil 2\pi/b\rceil$ and consider the set
\[X''=\lbrace x\in \mathbb{R}^2 ~|~ x=( \cos( {j}/{n'} )  , \sin ( {j}/{n'} ) ), j \in \mathbb{N} \rbrace.\]
Given $x\in X$ we define $g(x)\in\argmin_{x' \in X''} \| x-x' \|_2$ and $h(x)\in\argmin_{x' \in X''\setminus\{ g(x)\}} \| x-x' \|_2$, where ties are broken arbitrarily.
We set $X'=\{ g(x) ~|~ x \in X \} \cup \{ h(x) ~|~ x \in X \}$.
We note that the distance on the circle between two neighboring points in $X'$ is a multiple of $\frac{2\pi}{n'}$.
This implies that for any cone $C(V)$ between consecutive points in $X'$ with $P(V)>0$ we have $P(V)\geq 1/n'\geq b/7$ and $\| x- g(x)\|_2 \leq \frac{b}{2}$.
Further note that there are at most $ |X' |\leq 2n$ cones of this form.

The second step is to construct a separator $(u_s)_{s\leq m} \in \mathbb{R}^{m \times d}$: 
Let $\bar{v}\in \mathcal{F}_B$ be optimal for $(X, Y)$, i.e.,  $\gamma=\gamma(X,Y)=\gamma_{\bar{v}}$.
As in Lemma \ref{psetlem} construct $ \bar{v}' \in \mathcal{F}_B$ with $\E [ \langle \bar{v}'(z), x' \rangle ~|~ z\in C(V)]=\E [ \langle \bar{v}(z), x' \rangle ~|~ z\in C(V)] $ where $\bar{v}'$ is constant for any cone of the form $C(V)$.
Using the Chernoff bound (\ref{lem:chernoff}) we get with failure probability at most $2\exp(\frac{1}{8} \cdot \frac{b}{7} \cdot m')=2\exp(\frac{1}{224} \cdot \gamma \cdot m')$ that the number $n_V $ of points $w_{j, 0}$ in $C(V)$ lies in the interval $[\frac{P(V)m}{2}, 2P(V)m]$.
Now using $m'\geq 224 \gamma^{-1} \log(\frac{2n}{\delta})$ and applying a union bound we get that this holds for all cones of the form $C(V)$ with failure probability at most $2\delta$.
For $w_j \in C(V)$ we define $u_j=a_j\frac{\bar{v}'(w_j)}{\sqrt{m}}\cdot \frac{P(V)m}{2n_V}$.
Since $n_V \in [\frac{P(V)m}{2}, 2P(V)m]$ it follows that $\|u_j\|_2\leq \frac{\|\bar{v}'(w_j)\|_2}{\sqrt{m}} \leq \frac{1}{\sqrt{m}} $ and consequently $\| U \|_F \leq 1$.
Moreover we have
\begin{align*}
\sum_{s \in [m], w_{s, 0}\in C(V)} a_s u_s =P(V)m \cdot \frac{1}{2\sqrt{m}} \cdot \bar{v}'(V),
\end{align*}
where we set $\bar{v}'(V) $ to be equal to $\bar{v}'(z) $, which is constant for any $z \in C(V)$.

The third step is to prove that $U$ is a good separator for $(X, Y)$:
To this end, let $x \in X$ and $x'=g(x_i)$.

If $x_i=x'$ then
\begin{align*}
y_i  f^{(0)}_i(U)= & ~ y_i \frac{1}{\sqrt{m}}\sum_{s=1}^{m}a_s \langle u_s, x_i \rangle\mathbf{1}\left[ \langle x_i, w_{s, 0} \rangle >0 \right]\\
= & ~ y_i\frac{1}{\sqrt{m}}\sum_{V \subseteq X', x' \in V }\sum_{s \in [m], w_{s, 0}\in C(V)}a_s \langle u_s, x_i \rangle \\
= & ~ y_i\frac{1}{2m}\sum_{V \subseteq X', x' \in V}P(V)m \cdot \langle \bar{v}'(V), x_i \rangle \\
 = & ~ y_i\frac{1}{2}\E [ \langle \bar{v}(z), x_i \rangle \mathbf{1}\left[ \langle x_i,z \rangle >0 \right] ]\\
 =  & ~ y_i\frac{1}{2} \int \langle \bar{v}(z), x_i \rangle \mathbf{1}[ \langle x_i, z \rangle >0 ] \d \mu_{\mathcal{N}}(z) \geq \frac{\gamma}{2}.
\end{align*}
Otherwise if $x_i \neq x'$ then there is exactly one cone $C(V_1)$ with $z \in C(V_1)$ such that $\langle x', z \rangle<0$ and $\langle x_i, z \rangle>0$ and exactly one cone $C(V_2)$ with $z \in C(V_2)$ such that $\langle x', z \rangle>0$ and $\langle x, z \rangle<0$.
Recall that $P(V_i)=\frac{1}{n'}\leq b$ for $i=1, 2$.
We set $M=\{ V \subseteq [n'] ~|~ x' \in V , V \notin \{V_1, V_2\} \}$.
Then it holds that
\begin{align*}
y_i &f^{(0)}_i(U) = \frac{1}{\sqrt{m}}\sum_{s=1}^{m}y_i\langle u_s, x_i \rangle \mathbf{1}\left[ \langle x_i, w_{s, 0} \rangle >0 \right]\\
\geq & ~ \frac{1}{\sqrt{m}}\left(\sum_{V\in M}\sum_{s \in [m], w_{s, 0}\in C(V)}y_i\langle u_s, x_i \rangle  ~ - ~ \sum_{s \in [m], w_{s, 0}\in C(V_1)}|\langle u_s, x_i \rangle| \right)\\
\geq & ~ \frac{1}{\sqrt{m}}\left(\sum_{V\in M}\sum_{s \in [m], w_{s, 0}\in C(V)}y_i\langle u_s, x_i \rangle 
 +  \sum_{s \in [m], w_{s, 0}\in C(V_2)}|\langle u_s, x_i \rangle| - \sum_{s \in [m], w_{s, 0}\in C(V_2)}|\langle u_s, x_i \rangle|  - \frac{1}{2\sqrt{m}}P(V_1)m \right)\\
\geq & ~  \frac{1}{\sqrt{m}}\left( \frac{\sqrt{m}}{2}\E [ y_i\langle \bar{v}(z), x_i \rangle \mathbf{1}\left[ \langle x_i,z \rangle >0 \right] ]   ~ - ~ \frac{1}{2\sqrt{m}}P(V_2)m  ~ - ~ \frac{1}{2\sqrt{m}}P(V_1)m \right)\\
 = & ~ \frac{1}{2}\left(\E [ y_i\langle \bar{v}(z), x_i \rangle \mathbf{1}\left[ \langle x_i,z \rangle >0 \right] ]- 2b \right)
\geq  \frac{1}{2}\left(\gamma- \frac{\gamma}{2}\right)=\frac{\gamma}{4}.
\end{align*}
\end{proof}

%% file: arxivlogwidth_upper.tex
\section{Upper bound for log width}

We use the following initialization, see Definition \ref{def:initialization}: we set $m=2m'$ for some natural number $m'$.
Put $ w_{s, 0}=w_{s+m', 0}=\beta w_s'$ where $w_s' \sim \mathcal{N}(0, I_d), \beta \in \mathbb{R}$ is an appropriate scaling factor to be defined later and $a_i=1$ for $i<m'$ and $a_i=-1$ for $i\geq m'$.
We note that to simplify notations the $a_i$ are permuted compared to Definition \ref{def:initialization}, which does not make a difference.
Further note that $\frac{\partial f}{\partial w_s}=\frac{\partial f}{\partial w_s'}$.

The goal of this section is to show our main theorem:

\begin{theorem}\label{mainthm}
Given an error parameter $\varepsilon\in (0,1/10)$ and any failure probability $\delta \in (0, 1/10 )$, let $\rho= 2\cdot \gamma^{-1} \cdot \ln(4/\varepsilon) .$
Then if $$m = 2m' \geq 2 \gamma^{-2} \cdot 8\ln(2n/\delta),$$ $\beta=\frac{4\cdot 2\rho^2 n\sqrt{m}}{5\varepsilon\delta}$ and $\eta=1$ we have with probability at most $1 -3 \delta$ over the random initialization that $\frac{1}{T} \sum_{t =0}^{T-1} R(W_t) \leq \varepsilon$, where $T=\lceil {2 \rho^2}/{\varepsilon} \rceil$.
\end{theorem}

Before proving Theorem \ref{mainthm} we need some helpful lemmas.
Our first lemma shows that with high probability there is a good separator at initialization, similar to \cite{jt20}.

\begin{lem}\label{lem2.3}
If $m' \geq \frac{8\ln(2n/\delta)}{\gamma^2}$ then there exists $U \in \mathbb{R}^{m \times d}$ with $\|u_s\|_2 \leq \frac{1}{\sqrt{m}}$ for all $s\leq m$, and $\|  U \|_F\leq 1$, such that with probability at least $1- \delta$ it holds simultaneously for all $i\leq n$ that
\begin{align*}
y_if^{(0)}_i(U)\geq \frac{\gamma}{2}
\end{align*}
\end{lem}

\begin{proof}
We define $U$ by $ u_s=\frac{a_s}{\sqrt{m}}\bar{v}(w_{s,0})$.
Observe that
\begin{align*}
\mu_i=\E_{w \sim \mathcal{N}(0, I_d)}\left[y_i\langle \bar{v}(w), x_i \rangle]\mathbf{1}\left[ \langle x_i, w \rangle >0  \right]\right] \geq \gamma
\end{align*}
by assumption.
Further since $w_{s, 0}=w_{s+m', 0}= \beta w'_{s, 0} $ and $a_s^2=1$, we have $a_s u_s=a_{s+m'}u_{s+m'}$ for $s\leq m'$.
Also by Lemma \ref{psetlem} we can assume that $\bar{v}(w_{s, 0})=\bar{v}(w'_{s, 0})$.
Thus
\begin{align*}
y_i f^{(0)}_i(U)=\frac{1}{m'}\sum_{s=1}^{m'}y_i\langle \bar{v}(w_{s, 0}), x_i \rangle \mathbf{1}\left[ \langle x_i, w_{s, 0} \rangle >0 \right]
\end{align*}
is the empirical mean of i.i.d. random variables supported on $[-1,+1]$ with mean $\mu_i$. Therefore by Hoeffding’s inequality (Lemma~\ref{lem:hoeffding}), using $m' \geq \frac{8\ln(2n/\delta)}{\gamma^2}$ it holds that
\begin{align*}
\Pr [y_i f^{(0)}_i(U)\leq \frac{\gamma}{2}] 
\leq & ~ \Pr [|y_i f^{(0)}_i(U)-\mu_i|\geq \frac{\mu_i}{2}] \\
\leq & ~ 2\exp\left(- \frac{2\mu_i^2m'^2/4}{m'\cdot  4} \right) \\
\leq & ~ 2\exp\left(- \frac{\gamma^2m'}{8} \right) \leq \frac \delta n
\end{align*}
Applying the union bound proves the lemma.
\end{proof}

\begin{lem}\label{lem2.4}
With probability $1 - \delta$ it holds that $|\langle x_i, w_{s, 0} \rangle| > \frac{2 \rho^2}{\varepsilon \sqrt{m}}$ for all $i \in [n]$ and $s \in [m]$
\end{lem}

\begin{proof}
By anti-concentration of the Gaussian distribution (Lemma~\ref{lem:anti_gaussian}), we have for any $i$
\begin{align*}
\Pr[ | \langle x_i, w_{s, 0} \rangle| \leq \frac{2 \rho^2}{\varepsilon \sqrt{m}} ] = & ~ \Pr [ |\langle x_i, w'_{s, 0} \rangle| \leq \frac{2 \rho^2}{\beta \varepsilon \sqrt{m}} ]  \\
\leq & ~ \frac{2 \rho^2}{\beta \varepsilon \sqrt{m}} \frac{4}{5} \\
\leq & ~ \frac{\delta}{m n}.
\end{align*}
Thus applying the union bound proves the lemma.
\end{proof}

\begin{lem}\label{lem2.5}
For all $i \in [n]$ it holds that $f_i(W_0)=0 $
\end{lem}

\begin{proof}
Since $a_s=-a_{s+m'}$ we have
\begin{align*}
f_i(W_0)=\sum_{s=1}^m \frac{1}{\sqrt{ m } }a_s \phi\left(\langle w_{s,0}, x_i \rangle \right)
=\sum_{s=1}^{m'} \frac{1}{\sqrt{ m } }(a_s+a_{s+m'}) \phi\left(\langle w_{s,0}, x_i \rangle \right)
=0.
\end{align*}
\end{proof}

Further we need the following lemma proved in \cite{jt20}. 

\begin{lem}[Lemma 2.6 in \cite{jt20}]\label{lem2.6}
For any $t \geq 0$ and $\bar{W}$, if $\eta_t\leq 1$ then
\begin{align*}
\eta_t R(W_t) \leq \| W_t - \bar{W} \|_F^2 - \| W_{t+1} - \bar{W} \|_F^2 +2\eta_t R^{(t)}(\bar{W}).
\end{align*}
Consequently, if we use a constant step size $\eta \leq 1 $ for $ 0 \leq \tau < t$, then
\begin{align*}
\eta \sum_{\tau < t} R(W_\tau) \leq
\eta \sum_{\tau < t} R(W_\tau) + \| W_t - \bar{W} \|_F^2 \leq \| W_{0} - \bar{W} \|_F^2 + 2\eta\sum_{\tau < t} R^{(\tau)}(\bar{W}).
\end{align*}
\end{lem}

Now we are ready to prove the main theorem:

\begin{proof}[Proof of Theorem \ref{mainthm}]
With probability at least $1-2\delta$ there exists $U$ as in Lemma \ref{lem2.3} and also the statement of Lemma \ref{lem2.4} holds.
We set $\bar{W}=W_0+\rho U$.
First we show that for any $t < T$ and any $s\in [m]$ we have $\| w_{s, t}-w_{s, 0} \|_2 \leq \frac{2 \rho^2}{\varepsilon \sqrt{m}}$.
Observe that $|\ell'(v)|=|\frac{-e^{-v}}{1+e^{-v}}|\leq 1$ since $ e^{-v}>0$ for all $v \in \mathbb{R}$.
Thus for any $t \geq 0$ we have
\begin{align*}
\| w_{s, t}-w_{s, 0} \|_2 &\leq \sum_{\tau < t}\frac{1}{n}\sum_{i=1}^n|\ell'(y_i f_i(W_{\tau}))|\left\|\frac{\partial f_i}{\partial w_{s, t}}\right\|_2
\leq \sum_{\tau < t}\frac{1}{n}\sum_{i=1}^n 1 \cdot \frac{1}{\sqrt{m}}
\leq \frac{t}{\sqrt{m}}.
\end{align*}
Consequently we have $\| w_{s, t}-w_{s, 0} \|_2 \leq \frac{2 \rho^2}{\varepsilon \sqrt{m}}$ for $t< T = \lceil \frac{2 \rho^2}{\varepsilon} \rceil$.

Next we prove that for any $t < T$ we have $R^{(t)}(\bar{W})< \varepsilon/4$. Since $\ln(1+r) \leq r$ for any $r$, the logistic loss satisfies $\ell(z) = \ln(1 + \exp(-z)) \leq \exp(-z$), and it is sufficient to prove that for any $1 \leq i \leq n$ we have
\begin{align*}
y_i \langle \nabla f_i(W_t), \bar{W} \rangle \geq \ln \left( \frac{\varepsilon}{4} \right).
\end{align*}
Note that
\begin{align*}
y_i\langle \nabla f_i(W_t), \bar{W} \rangle 
&= y_i \langle \nabla f_i(W_t), W_0 \rangle+y_i \rho\langle \nabla f_i(W_t), U \rangle \\
&=y_i \langle \nabla f_i(W_t), W_0 \rangle+y_i \langle \nabla f_i(W_0), W_0 \rangle-y_i \langle \nabla f_i(W_0), W_0 \rangle+
y_i\rho \langle \nabla f_i(W_t), U \rangle  \\
&=y_i \langle \nabla f_i(W_0), W_0 \rangle+y_i \langle \nabla f_i(W_t)-\nabla f_i(W_0), W_0 \rangle+
y_i \rho\langle \nabla f_i(W_t), U \rangle .
\end{align*}
For the first term we have $y_i \langle \nabla f_i(W_0), W_0 \rangle=y_i f_i(W_0)=0 $ by Lemma \ref{lem2.5}.
For the second term we note that $|\langle x_i, w_{s , 0} \rangle - \langle x_i, w_{s, t} \rangle|= |\langle x_i, w_{s, 0} - w_{s, t} \rangle| \leq \| x_i \|_2 \| w_{s, 0} - w_{s, t} \|_2 \leq \frac{2 \rho^2}{\varepsilon \sqrt{m}}$.
Thus $\mathbf{1}\left[ \langle x_i, w_{s, 0} \rangle >0  \right]\neq \mathbf{1}\left[ \langle x_i, w_{s, t} \rangle >0  \right]$ can only hold if $|\langle x_i, w_{s,0} \rangle | \leq \frac{2 \rho^2}{\varepsilon \sqrt{m}}$ which is false for all $i, s$ by Lemma \ref{lem2.4}.
Hence it holds that
\begin{align*}
\frac{\partial f_i}{\partial w_{s, t}}= \frac{1}{\sqrt{m}} a_s \mathbf{1}\left[ \langle x_i, w_{s, t} \rangle >0  \right] x_i= \frac{1}{\sqrt{m}} a_s \mathbf{1}\left[ \langle x_i, w_{s, 0} \rangle >0  \right] x_i =\frac{\partial f_i}{\partial w_{s, 0}}
\end{align*}
and consequently $\nabla f_i(W_t)=\nabla f_i(W_0)$.
It follows for the second term that
\begin{align*}
y_i \langle \nabla f_i(W_t)-\nabla f_i(W_0), W_0 \rangle=0.
\end{align*}
Moreover by Lemma \ref{lem2.3} for the third term it follows
\begin{align*}
y_i \rho\langle \nabla f_i(W_t), U \rangle =y_i \rho\langle \nabla f_i(W_0), U \rangle \geq \rho\frac{\gamma}{2}.
\end{align*}
Thus $y_i \langle \nabla f_i(W_t), \bar{W} \rangle \geq \rho \frac{\gamma}{2}\geq \ln(4/\varepsilon) $ since $\rho=2\gamma^{-1} \cdot \ln(4/\varepsilon) $.
Consequently it holds that $R^{(t)}(\bar{W})< \varepsilon/4$.

Now using $T=\lceil \frac{2\rho^2}{\varepsilon}  \rceil$ applying Lemma \ref{lem2.6} with step size $\eta = 1$ gives us the desired result:
\begin{align*}
\frac{1}{T}\sum_{t < T} R(W_t) & \leq  \frac{\| W_0-\bar{W}\|_F^2}{T}+ \frac{2}{T}\sum_{\tau <T} R^{(t)}(\bar{	W}) \\
& = \frac{\| \rho U \|_F^2}{T}+ \frac{2}{T}\sum_{\tau <T} R^{(t)}(\bar{	W}) \\
\leq & ~\frac{\varepsilon}{2}+\frac{\varepsilon}{2} \\
\leq & ~ \varepsilon.
\end{align*}
\end{proof}

%% file: arxivanalysis.tex
\section{Width under squared loss}\label{sec:analysis_concentration}

\subsection{Analysis: achieving concentration}\label{sec:analysis_concentration'}

We first present a high-level overview. In Lemma~\ref{lem:3.1}, we prove that the initialization (kernel) matrix $H$ is close to the neural tangent kernel (NTK). In Lemma~\ref{lem:3.2}, we bound the spectral norm change of $H$, given that the weight matrix $W$ does not change much. 
In Section~\ref{sec:cont} we consider the (simplified) continuous case, where the learning rate is infinitely small. This provides most of the intuition.
In Section~\ref{sec:dis} we consider the discretized case where we have a finite learning rate. This follows the same intuition as in the continuous case, but we need to deal with a second order term given by the gradient descent algorithm.

The high level intuition of the proof is to recursively prove the following: 
\begin{enumerate}
	\item The weight matrix does not change much.
	\item Given that the weight matrix does not change much, the prediction error decays exponentially. 
\end{enumerate}

Given (1) we prove (2) as follows. The intuition is that the kernel matrix does not change much, since the weights do not change much, and it is close to the initial value of the kernel matrix, which is in turn close to the NTK matrix (involving the entire Gaussian distribution rather than our finite sample), that has a lower bound on its minimum eigenvalue. Thus, the prediction loss decays exponentially.

Given (2) we prove (1) as follows. Since the prediction error decays exponentially, one can show that the change in weights is upper bounded by the prediction loss, and thus the change in weights also decays exponentially and the total change is small.

\subsection{Bounding the difference between the continuous and discrete case}

In this section, we show that when the width $m$ is sufficiently large,
then the continuous version and discrete version of the Gram matrix of 
the input points are spectrally close. 
We prove the following Lemma, which is a variation of Lemma 3.1 in \cite{sy19} and also of Lemma 3.1 in \cite{dzps19}.
\begin{lemma}[Formal statement of Lemma~\ref{lem:3.1:intro}]
\label{lem:3.1}
Let $\{ w_1, w_2, \ldots, w_m \} \subset\R^d$ denote a collection of vectors constructed as in  Definition~\ref{def:initialization}.
We define $H^{\cts}, H^{\dis} \in \R^{n \times n}$ as follows
\begin{align*}
H^{\cts}_{i,j} := & ~ \E_{w \sim \N(0,I)} \left[ x_i^\top x_j {\bf 1}_{ w^\top x_i \geq 0, w^\top x_j \geq 0 } \right] , \\ 
H^{\dis}_{i,j} := & ~ \frac{1}{m} \sum_{r=1}^m \left[ x_i^\top x_j {\bf 1}_{ w_r^\top x_i \geq 0, w_r^\top x_j \geq 0 } \right].
\end{align*}
Let $\lambda = \lambda_{\min} (H^{\cts}) $. If $m_0 = \Omega( \lambda^{-2} n^2\log (nB/\delta) )$, we have that  
\begin{align*}
\| H^{\dis} - H^{\cts} \|_F \leq \frac{ \lambda }{4}, \mathrm{~and~} \lambda_{\min} ( H^{\dis} ) \geq \frac{3}{4} \lambda
\end{align*}
holds with probability at least $1-\delta$.
\end{lemma}

\begin{proof}
For every fixed pair $(i,j)$,
$H_{i,j}^{\dis,b}$ $(b \in [B])$ is an average of independent random variables,
i.e., 
\begin{align*}
H_{i,j}^{\dis,b}=~\frac {1}{m_0}\sum_{r=1}^{m_0} x_i^\top x_j\mathbf{1}_{w_{r, b}^\top x_i\geq 0,w_{r, b}^\top x_j\geq 0}, 
\end{align*}
and $H_{i,j}^{\dis}$ is the average of all sampled Gaussian vectors: 
\begin{align*}
H_{i,j}^{\dis}=\frac{1}{B}\sum_{b=1}^{B}H_{i,j}^{\dis, b} = \frac {1}{m}\sum_{r=1}^{m_0}\sum_{b=1}^{B} x_i^\top x_j\mathbf{1}_{w_{r, b}^\top x_i\geq 0,w_{r, b}^\top x_j\geq 0}.
\end{align*}

The expectation of $H_{i,j}^{\dis}$ is
\begin{align*}
\E [ H_{i,j}^{\dis,b} ]
= & ~\frac {1}{m}\sum_{r=1}^{m_0} \E_{w_{r, b}\sim {\N}(0,I_d)} \left[ x_i^\top x_j\mathbf{1}_{w_{r, b}^\top x_i\geq 0,w_{r, b}^\top x_j\geq 0} \right]\\
= & ~\E_{w\sim {\N}(0,I_d)} \left[ x_i^\top x_j\mathbf{1}_{w^\top x_i\geq 0,w^\top x_j\geq 0} \right]
= ~ H_{i,j}^{\cts}.
\end{align*}
Therefore,
\begin{align*}
\E [ H_{i,j}^{\dis, b} ] = \E [ H_{i,j}^{\dis} ] = H_{i,j}^{\cts}. 
\end{align*}

For $r\in [m_0]$,
let $z_r=\frac {1}{m_0}x_i^\top x_j \mathbf{1}_{w_{r, b}^\top x_i\geq 0,w_{r, b}^\top x_j\geq 0}$.
Then $z_r$ is a random function of $w_{r, b}$,
and hence, the  $\{z_r\}_{r\in [m_0]}$ are mutually independent.
Moreover,
$-\frac {1}{m_0}\leq z_r\leq \frac {1}{m_0}$.
By Hoeffding's inequality (Lemma \ref{lem:hoeffding}), we have that for all $t>0$,
\begin{align*}
\Pr \left[ | H_{i,j}^{\dis,b} - H_{i,j}^{\cts} | \geq t \right]
\leq & ~ 2\exp \Big( -\frac{2t^2}{4/m_0} \Big)  =  ~ 2\exp(-m_0 t^2/2).
\end{align*}

Setting $t=( \frac{1}{m_0} 2 \log (2n^2 B /\delta) )^{1/2}$, 
we can apply a union bound over $b$ and all pairs $(i,j)$ to get that with probability at least $1-\delta$,
for all $i,j\in [n]$,
\begin{align*}
|H_{i,j}^{\dis} - H_{i,j}^{\cts}|
\leq \Big( \frac{2}{m_0}\log (2n^2B/\delta) \Big)^{1/2}
\leq 4 \Big( \frac{\log ( nB/\delta ) }{m_0} \Big)^{1/2}.
\end{align*}
Thus, we have
\begin{align*}
\|H^{\dis} - H^{\cts}\|^2 
\leq & ~ \|H^{\dis} - H^{\cts}\|_F^2 \\
 = & ~ \sum_{i=1}^n\sum_{j=1}^n |H_{i,j}^{\dis} - H_{i,j}^{\cts}|^2 \\
 \leq & ~ \frac{1}{m_0} 16n^2\log (nB/\delta).
\end{align*}
Hence, if $m_0 = \Omega( \lambda^{-2} n^2\log (nB/\delta) )$, we have the desired result.
 \end{proof}

\subsection{Bounding changes of \texorpdfstring{$H$}{H} when \texorpdfstring{$w$}{w} is in a small ball}

In this section, we bound the change of $H$ when $w$ is in a small ball.
We define the event
\begin{align*}
A_{i,r} = \left\{ \exists u : \| u - \wt{w}_r \|_2 \leq R, {\bf 1}_{ x_i^\top \wt{w}_r \geq 0 } \neq {\bf 1}_{ x_i^\top u \geq 0 } \right\}.
\end{align*}
Note this event happens if and only if $| \wt{w}_r^\top x_i | < R$. Recall that $\wt{w}_r \sim \N(0,I)$. By anti-concentration of the Gaussian distribution (Lemma~\ref{lem:anti_gaussian}), we have
\begin{align}\label{eq:Air_bound}
\Pr[ A_{i,r} ] = \Pr_{ z \sim \N(0,1) } [ | z | < R ] \leq \frac{ 2 R }{ \sqrt{2\pi} }.
\end{align}

We prove the following perturbation Lemma, which is a variation of Lemma 3.2 in \cite{sy19} and Lemma 3.2 in \cite{dzps19}.
\begin{lemma}[Formal version of Lemma~\ref{lem:3.2:intro}]\label{lem:3.2}
Let $R \in (0,1)$. Let $\{w_1, w_2, \ldots, w_m\}$ denote a collection of weight vectors constructed as in  Definition~\ref{def:initialization}. 
For any set of weight vectors $\wt{w}_1, \ldots, \wt{w}_m \in \R^d$ that satisfy that for any $r\in [m]$, $\| \wt{w}_r - w_r \|_2 \leq R$, consider the map $H : \R^{m \times d} \rightarrow \R^{n \times n}$ defined by 
\begin{align*}
    H(w)_{i,j} =  \frac{1}{m} x_i^\top x_j \sum_{r=1}^m {\bf 1}_{ \wt{w}_r^\top x_i \geq 0, \wt{w}_r^\top x_j \geq 0 } .
\end{align*}
Then we have that 
\begin{align*}
\| H (w) - H(\wt{w}) \|_F < 2 n R,
\end{align*}
holds with probability at least $1-n^2 \cdot B \cdot \exp(-m_0 R /10)$.
\end{lemma}
\begin{proof}

The random variable we care about is
\begin{align*}
 ~ \sum_{i=1}^n \sum_{j=1}^n | H(\wt{w})_{i,j} - H(w)_{i,j} |^2 \leq & ~ \frac{1}{m^2} \sum_{i=1}^n \sum_{j=1}^n \left( \sum_{r=1}^m {\bf 1}_{ \wt{w}_r^\top x_i \geq 0, \wt{w}_r^\top x_j \geq 0} - {\bf 1}_{ w_r^\top x_i \geq 0 , w_r^\top x_j \geq 0 } \right)^2 \\
= & ~ \frac{1}{m^2} \sum_{i=1}^n \sum_{j=1}^n  \Big( \sum_{r=1}^m s_{r,i,j} \Big)^2 ,
\end{align*}
where the last step follows from defining, for each $r,i,j$, 
\begin{align*}
s_{r,i,j} :=  {\bf 1}_{ \wt{w}_r^\top x_i \geq 0, \wt{w}_r^\top x_j \geq 0} - {\bf 1}_{ w_r^\top x_i \geq 0 , w_r^\top x_j \geq 0 } .
\end{align*} 

Now consider that $i,j$ are fixed. We simplify $s_{r,i,j}$ to $s_r$.

Then $s_r$ is a random variable that only depends on $w_r$.

If  events $\neg A_{i,r}$ and $\neg A_{j,r}$ happen,
then 
\begin{align*}
\left| {\bf 1}_{ \wt{w}_r^\top x_i \geq 0, \wt{w}_r^\top x_j \geq 0} - {\bf 1}_{ w_r^\top x_i \geq 0 , w_r^\top x_j \geq 0 } \right|=0.
\end{align*}
If   $A_{i,r}$ or $A_{j,r}$ happens,
then 
\begin{align*}
\left| {\bf 1}_{ \wt{w}_r^\top x_i \geq 0, \wt{w}_r^\top x_j \geq 0} - {\bf 1}_{ w_r^\top x_i \geq 0 , w_r^\top x_j \geq 0 } \right|\leq 1.
\end{align*}
Thus we have {
\begin{align*}
 \E_{\wt{w}_r}[s_r]\leq \E_{\wt{w}_r} \left[ {\bf 1}_{A_{i,r}\vee A_{j,r}} \right] 
 \leq & ~ \Pr[A_{i,r}]+\Pr[A_{j,r}] \\
 \leq & ~ \frac {4 R}{\sqrt{2\pi}} \\
 \leq & ~ 2 R,
\end{align*}}
and {
\begin{align*}
    \E_{\wt{w}_r} \left[ \left(s_r-\E_{\wt{w}_r}[s_r] \right)^2 \right]
    = & ~ \E_{\wt{w}_r}[s_r^2]-\E_{\wt{w}_r}[s_r]^2 \\
    \leq & ~ \E_{\wt{w}_r}[s_r^2]\\
    \leq & ~\E_{\wt{w}_r} \left[ \left( {\bf 1}_{A_{i,r}\vee A_{j,r}} \right)^2 \right] \\
     \leq & ~ \frac {4R}{\sqrt{2\pi}} \\
     \leq  &~ 2 R .
\end{align*}}
We also have $|s_r|\leq 1$. 

Fix $b \in B$ and consider $s_{1, b}, \ldots, s_{m_0, b}$. Applying Bernstein's inequality (Lemma~\ref{lem:bernstein}), we get that for all $t>0$,{
\begin{align*}
     ~ \Pr \left[\sum_{r=1}^{m_0} s_{r, b} \geq 2 m_0 R +m_0 t \right] 
    \leq & ~ \Pr \left[\sum_{r=1}^{m_0} (s_{r, b}-\E[s_{r, b}])\geq m_0 t \right]\\
    \leq & ~ \exp \left( - \frac{ m_0^2t^2/2 }{ 2 m_0 R   + m_0 t / 3 } \right).
\end{align*}}
Choosing $t = R$, we get that 
\begin{align*}
    \Pr \left[\sum_{r=1}^{m_0} s_{r, b} \geq 3 m_0 R  \right]
    \leq & ~ \exp \left( -\frac{ m_0^2  R^2 /2 }{ 2 m_0 R + m_0  R /3 } \right) 
     \leq  ~ \exp \left( - m_0 R / 10 \right) .
\end{align*}
Thus, we have 
\begin{align*}
\Pr \left[ \frac{1}{m_0} \sum_{r=1}^{m_0} s_r \geq 3  R \right] \leq \exp( - m_0 R /10 ).
\end{align*}
Next, taking a union bound over $B$ such events, 
\begin{align*}
\Pr \left[ \frac{1}{m} \sum_{r=1}^{m} s_r \geq 3  R \right] 
=  & ~
\Pr \left[ \frac{1}{B} \sum_{b=1}^B \frac{1}{m_0} \sum_{r=1}^{m_0} s_{r, b} \geq 3  R \right] \\
= & ~
\Pr \left[ \sum_{b=1}^B \frac{1}{m_0} \sum_{r=1}^{m_0} s_{r, b} \geq 3  R \cdot B \right] \leq  ~ B \cdot \exp( - m_0 R /10 ).
\end{align*}
This completes the proof.
\end{proof}

\section{Analysis: convergence}\label{sec:analysis_convergence}

\subsection{The continuous case}
\label{sec:cont}

We first consider the continuous case, in which the learning rate $\eta$ is sufficiently small. This provides an intuition for the discrete case.

For any $s \in [0,t]$, we define the kernel matrix $H(s) \in \R^{n \times n}$:  
\begin{align*}
H(s)_{i,j} = \frac{1}{m} \sum_{r=1}^m x_i^\top x_j {\bf 1}_{ w_r(s)^\top x_i \geq 0, w_r(s)^\top x_j \geq 0 }.
\end{align*} 
We consider the following dynamics of a gradient update: 
\begin{align}
\label{eq:gradient-dynamic-continuous}
\frac{\partial W(t)}{\partial t} = \frac{1}{n^2}\frac{\partial L(W(t))}{\partial W(t)} 
\end{align}
The dynamics of prediction can be written as follows, which is a simple calculation: 
\begin{fact}\label{fact:dudt}
$
\frac{\d}{\d t} u(t)= \frac{m}{n^2}H(t) \cdot (y-u(t)) .
$
\end{fact}

\begin{proof}
	For each $i\in [n]$,
	we have
	
	\begin{align*}
	  ~ \frac{\d}{\d t} u_i(t) 
	= & ~ \sum_{r=1}^m \left\langle \frac{\partial f(W(t),a,x_i)}{\partial w_r(t)},\frac{\d w_r(t)}{\d t} \right\rangle\\
	= & ~ \sum_{r=1}^m \left\langle \frac{\partial f(W(t),a,x_i)}{\partial w_r(t)},-\frac{1}{n^2}\cdot \frac{\partial L(w(t),a)}{\partial w_r(t)} \right\rangle\\
	= & ~ \frac{1}{n^2}\sum_{r=1}^m \Big\langle \frac{\partial f(W(t),a,x_i)}{\partial w_r(t)},
	- \sum_{j=1}^n ( f(W,x_j,a_r) - y_i ) a_r x_j {\bf 1}_{ w_r^\top x_j \geq 0 } \Big\rangle\\
	= & ~ \frac{1}{n^2}\sum_{j=1}^n (y_j-u_j(t)) \sum_{r=1}^{m}\cdot \left\langle \frac{\partial f(W(t),a,x_i)}{\partial w_r(t)},\frac{\partial f(W(t),a,x_j)}{\partial w_r(t)} \right\rangle\\
	= & ~ \sum_{j=1}^n (y_j-u_j(t)) \frac{m}{n^2} \cdot H(t)_{i,j}
	\end{align*}
	where the first step follows from the chain rule, the second step follows from Eq.~\eqref{eq:gradient-dynamic-continuous}, the third step uses Eq. \eqref{eq:gradient1}, the fourth step uses Eq. \eqref{eq:relu_derivative},  and the last step uses the definition of the matrix $H$.
\end{proof}

\begin{lemma}\label{lem:3.3}
Suppose for $0 \leq s \leq t$, $\lambda_{\min} ( H( w(s) ) ) \geq \lambda / 2$. Let $D_{\cts}$ be defined as
$
D_{\cts} := \frac{ \sqrt{n} \| y - u(0) \|_2 }{ m \lambda }.
$
Then we have 
\begin{align*}
1. & ~ & \| w_r(t) - w_r(0) \|_2 \leq & ~ D_{\cts} , \forall r \in [m], \\
2. & ~ & \| y - u(t) \|_2^2 \leq  &~ \exp( - \lambda t ) \cdot \| y - u(0) \|_2^2.
\end{align*}
\end{lemma}

\begin{proof}
	
	Recall that we can write the dynamics of prediction as 
	\begin{align*}
	\frac{ \d }{ \d t} u(t) = \frac{m}{n^2} \cdot H(t) \cdot ( y - u(t) ) .
	\end{align*}
	We can calculate the loss function dynamics
	\begin{align*}
	\frac{\d }{ \d t } \| y - u(t) \|_2^2
	= & ~ - 2 ( y - u(t) )^\top \cdot \frac{m}{n^2} \cdot H(t) \cdot ( y - u(t) ) \\
	\leq & ~ - \frac{m}{n^2} \lambda \| y - u(t) \|_2^2 .
	\end{align*}
	
	Thus we have $\frac{\d}{ \d t} ( \exp( \frac{m}{n^2} \lambda t) \| y - u(t) \|_2^2 ) \leq 0$ and that $\exp( \frac{m}{n^2} \lambda t ) \| y - u(t) \|_2^2$ is a decreasing function with respect to $t$.
	
	Using this fact, we can bound the loss by
	\begin{align}\label{eq:yut}
	\| y - u(t) \|_2^2 \leq \exp( - \frac{m}{n^2} \lambda t ) \| y - u(0) \|_2^2.
	\end{align}

	Now, we can bound the gradient norm. For $0 \leq s \leq t$,
	\begin{align}\label{eq:gradient_bound}
	 \left\| \frac{ \d }{ \d s } w_r(s) \right\|_2 
	= & ~ \frac{1}{n^2}\left\| \sum_{i=1}^n (y_i - u_i) \cdot a_r x_i \cdot {\bf 1}_{ w_r(s)^\top x_i \geq 0 } \right\|_2 \notag\\
	\leq & ~ \frac{1}{n^2}\sum_{i=1}^n | y_i - u_i(s) | \notag\\
	\leq & ~\frac{1}{n^{3/2}}\| y - u(s) \|_2 \\
	\leq & ~ \frac{1}{n^{3/2}} \exp( - \frac{m}{n^2} \lambda s ) \| y - u(0) \|_2.\notag 
	\end{align}
	where the first step follows from Eq. \eqref{eq:gradient1} and Eq.~\eqref{eq:gradient-dynamic-continuous}, the second step follows from the triangle inequality and $a_r=\pm 1$ for $r\in [m]$ and $\|x_i\|_2=1$ for $i\in [n]$, the third step follows from the Cauchy-Schwarz inequality, and the last step follows from Eq. \eqref{eq:yut}.
	
	Integrating the gradient, we can bound the distance from the initialization
	\begin{align*}
	\| w_r(t) - w_r(0) \|_2 \leq & ~ \int_0^t \left\| \frac{\d}{ \d s} w_r(s) \right\|_2 \d s \\
	\leq & ~ \frac{ \sqrt{n} \| y - u(0) \|_2 }{ m \lambda } .
	\end{align*}
\end{proof}

\begin{lemma}\label{lem:3.4}
If $D_{\cts}<R$.
then for all $t\geq 0$,
$\lambda_{\min}(H(t))\geq \frac{1}{2}\lambda$.
Moreover,
\begin{align*}
1. & ~ & \|w_r(t)-w_r(0)\|_2 \leq & ~ D_{\cts}, \forall r \in [m], \\
2. & ~ & \|y-u(t)\|_2^2 \leq & ~ \exp(-\frac{m}{n^2}  \lambda t) \cdot \|y-u(0)\|_2^2.
\end{align*}
\end{lemma}

\begin{proof}
	Assume the conclusion does not hold at time $t$.
	We argue that there must be some $s\leq t$ so that $\lambda_{\min}( H ( s ) )<\frac {1}{2}\lambda$.

	If $\lambda_{\min}( H ( t  ) )<\frac {1}{2}\lambda$, then we can simply take $s=t$.
	
	Otherwise since the conclusion does not hold, there exists $r$ so that 
	\begin{align*}
	\|w_r(t)-w_r(0)\|\geq D_{\cts}
	\end{align*}
	or
	\begin{align*}
	\|y-u(t)\|_2^2>\exp(- \frac{m}{n^2}  \lambda t)\|y-u(0)\|_2^2.
	\end{align*}
	
	Then by Lemma \ref{lem:3.3}, there exists $s\leq t$ such that 
	\begin{align*}
	\lambda_{\min}( H ( s  ) )<\frac {1}{2}\lambda.
	\end{align*}
	
	By Lemma \ref{lem:3.2}, there exists $t_0> 0$ defined as
	\begin{align*}
	t_0=\inf \left\{ t>0 : \max_{r\in [m]} \|w_r(t)-w_r(0)\|_2^2\geq R \right\}.
	\end{align*}
	
	Thus at time $t_0$, there exists $r\in [m]$ satisfying $\|w_r(t_0)-w_r(0)\|_2^2=R$.
	
	By Lemma \ref{lem:3.2},
	\begin{align*}
	\lambda_{\min}(H(t'))\geq \frac {1}{2}\lambda , \forall t'\leq t_0.
	\end{align*}

	However, by Lemma \ref{lem:3.3}, this implies 
	\begin{align*}
	\|w_r(t_0)-w_r(0)\|_2\leq D_{\cts}<R,
	\end{align*}
	which is a contradiction.
\end{proof}

Combining Lemma~\ref{lem:3.3} and Lemma~\ref{lem:3.4}, we get that for a linear convergence to hold, it suffices to guarantee that

\begin{align*}
\| w_r(t+1) - w_r(0) \|_2 \leq \frac{ 4 \sqrt{n} \| y - u (0) \|_2 }{ m \lambda } < R 
\end{align*}
which implies
\begin{align*}
\frac{ 4 \sqrt{n} \sqrt{n} }{ m \lambda } < \frac{\lambda}{n} \Rightarrow m \geq O(n^2\lambda^{-2})
\end{align*}
Note that the first step holds since $\|y - u(0)\|_2 = O(\sqrt{n})$ (see Claim~\ref{cla:yu0}).

\subsection{The discrete case}
\label{sec:dis}

We next move to the discrete case. The major difference from the continuous case is that the learning rate is not negligible and there is a second order term for gradient descent which we need to handle.

\begin{theorem}[Formal version of Theorem~\ref{thm:main_informal}]\label{thm:main_formal}
Suppose there are $n$ input data points in $d$-dimensional space. 
Recall that $\lambda=\lambda_{\min}(H^{\cts})>0$.
Suppose the width of the neural network  satisfies that 
\begin{align*}
m = \Omega( \lambda^{-2} n^2 \log^3 (n/\delta) ).
\end{align*}
We initialize $W \in \R^{d \times m}$ and $a \in \R^m$ as in Definition~\ref{def:initialization}, and we set the step size, also called the learning rate, to be
\begin{align*}
\eta = O( \lambda / (n^2 m) ).
\end{align*}  
Then with probability at least $1-\delta$ over the random initialization, we have for $k = 0,1,2,\ldots$ that 
\begin{align}\label{eq:quartic_condition}
\| u (t) - y \|_2^2 \leq ( 1 - m \eta \lambda / 2 )^k \cdot \| u (0) - y \|_2^2.
\end{align}
Further, for any accuracy parameter $\epsilon \in (0,1)$, if we choose the number of iterations 
\begin{align*}
T = \Theta \Big( \frac{ \log (n/\epsilon) }{ m \eta \lambda } \Big) = \lambda^{-2} n^2 \log (n/\epsilon),
\end{align*}
then
\begin{align*}
\| u (T) - y \|_2^2 \leq \epsilon.
\end{align*}
\end{theorem}

\paragraph{Correctness}
We prove Theorem \ref{thm:main_formal} by induction.
The base case is $i=0$ and it is trivially true.
Assume for $i=0,\ldots,k$ we have proved Eq. \eqref{eq:quartic_condition} to be true.
We want to show that Eq. \eqref{eq:quartic_condition} holds for $i=k+1$.

From the induction hypothesis,
we have the following Lemma (see proof in Section \ref{sec:missing_proof}) stating that the weights should not change too much. 
Note that the Lemma is a variation of Corollary 4.1 in \cite{dzps19}.
\begin{lemma}\label{lem:4.1}
If Eq. \eqref{eq:quartic_condition} holds for $i = 0, \ldots, k$, then we have for all $r\in [m]$
\begin{align*}
\| w_r(t+1) - w_r(0) \|_2 \leq \frac{ 4 \sqrt{n} \| y - u (0) \|_2 }{ m \lambda } := D.
\end{align*}
\end{lemma}

Next, we calculate the difference of predictions between two consecutive iterations, analogous to the $\frac{\d u_i(t)}{ \d t }$ term in Fact \ref{fact:dudt}.
For each $i \in [n]$, we have
{
\begin{align*}
u_i(t+1) - u_i(t) 
= & ~ \sum_{r=1}^m a_r \cdot \left( \phi( w_r(t+1)^\top x_i ) - \phi(w_r(t)^\top x_i ) \right) \\
= & ~ \sum_{r=1}^m a_r \cdot z_{i,r} .
\end{align*}
}
where
\begin{align*}
z_{i,r} :=  \phi \left( \Big( w_r(t) - \eta \frac{ \partial L( W(t) ) }{ \partial w_r(t) } \Big)^\top x_i \right) - \phi ( w_r(t)^\top x_i ) .
\end{align*}

Here we divide the right hand side into two parts. $v_{1,i}$ represents the terms for which the pattern does not change, while $v_{2,i}$ represents the terms for which the pattern may change. For each $i \in [n]$,
we define the set $S_i\subset [m]$ as
\begin{align*}
    S_i:= & ~ \{r\in [m]:\forall 
    w\in \mathbb{R}^d \text{ s.t. }   \|w-w_r(0)\|_2\leq R, 
     ~ \mathbf{1}_{w_r(0)^\top x_i\geq 0}=\mathbf{1}_{w^\top x_i\geq 0}\}.
\end{align*}
Then we define $v_{1,i}$ and $v_{2,i}$ as follows
{
\begin{align*}
v_{1,i} : = & ~ \sum_{r \in S_i} a_r z_{i,r}, \\
v_{2,i} : = & ~ \sum_{r \in \ov{S}_i} a_r z_{i,r}.
\end{align*} 
}

Define $H$ and $H^{\bot} \in \R^{n \times n}$ as
\begin{align}
H(t)_{i,j} := & ~ \frac{1}{m} \sum_{r=1}^m x_i^\top x_j {\bf 1}_{ w_r(t)^\top x_i \geq 0, w_r(t)^\top x_j \geq 0 } , \label{eq:def_H(t)}\\
H(t)^{\bot}_{i,j} := & ~ \frac{1}{m} \sum_{r\in \ov{S}_i} x_i^\top x_j {\bf 1}_{ w_r(t)^\top x_i \geq 0, w_r(t)^\top x_j \geq 0 } \label{eq:def_H(t)^bot}.
\end{align}
and
\begin{align*}
C_1 := & ~ -2 \eta (y - u(t))^\top H(t) ( y - u(t) ) , \\
C_2 := & ~ + 2 \eta ( y - u(t) )^\top H(t)^{\bot} ( y - u(t) ) , \\
C_3 := & ~ - 2 ( y - u(t) )^\top v_2 , \\
C_4 := & ~ \| u (t+1) - u(t) \|_2^2 . 
\end{align*}

Then we have (the proof is deferred to Section \ref{sec:missing_proof})
\begin{claim}\label{cla:inductive_claim}
\begin{align*}
\| y - u(t+1) \|_2^2 = \| y - u(t) \|_2^2 + C_1 + C_2 + C_3 + C_4.
\end{align*}
\end{claim}

Applying Claim~\ref{cla:C1}, \ref{cla:C2}, \ref{cla:C3} and \ref{cla:C4} gives
\begin{align*}
  \| y - u(t+1) \|_2^2 
\leq & ~ \| y - u(t) \|_2^2 
  \cdot ( 1 - m \eta \lambda + 8 m \eta n R  + 8 m \eta n R  + m^2 \eta^2 n^2 ).
\end{align*}

\paragraph{Choice of $\eta$ and $R$.}

Next, we want to choose $\eta$ and $R$ such that
\begin{equation}\label{eq:choice_of_eta_R}
( 1 - m \eta \lambda + 8 m \eta n R  + 8 m \eta n R  + m^2 \eta^2 n^2 ) \leq (1- m \eta\lambda/2) .
\end{equation}

If we set $\eta=\frac{\lambda }{4n^2 m }$ and $R=\frac{\lambda}{64n}$, we have 
\begin{align*}
8 \eta n R  + 8 \eta n R = 16 \eta n R \leq  \eta \lambda /4 ,
\mathrm{~~~and~~~} m^2 \eta^2 n^2 \leq m \eta \lambda / 4.
\end{align*}
This implies
\begin{align*}
\| y - u(t+1) \|_2^2 \leq & ~ \| y - u(t) \|_2^2 \cdot ( 1 - m \eta \lambda / 2 )
\end{align*}
holds with probability at least $1-\poly(n,B) \cdot \exp(-mR/10)$. 
\paragraph{Over-parameterization size, lower bound on $m$.}

We require 
\begin{align*}
 D=  \frac{ 4 \sqrt{n} \| y - u(0) \|_2 }{ m\lambda } < R = \frac{\lambda}{64n} ,
\end{align*}
and
\begin{align*}
  \poly(n,B) \cdot \exp( - m R / 10 ) \leq  \delta .
\end{align*}
By Claim \ref{cla:yu0}, it is sufficient to choose 
\begin{align*}
 m = \Omega( \lambda^{-2} n^2 \log(m/\delta)\log^2(n/\delta) ).
\end{align*}

%% file: arxivtechnical.tex
\section{Technical claims }\label{sec:missing_proof}

\begin{table*}[!htbp]\caption{{\bf Nt.} stands for notation. 
		$m$ is the width of the neural network. $n$ is the number of input data points. $\delta$ is the failure probability.}
	\centering
	\begin{tabular}{ | l| l| l| l| } 
		\hline
		{\bf Nt.} & {\bf Choice} & {\bf Place} & {\bf Comment} \\\hline
		$\lambda$ & $:= \lambda_{\min}(H^{\cts}) $ & Lemma~\ref{lem:3.1}  & Data-dependent \\ \hline
		$R$ & $\lambda/n$ & Eq.~\eqref{eq:choice_of_eta_R} & Maximal allowed movement of weight \\ \hline
		$D_{\cts}$ & $\frac{ \sqrt{n} \| y - u(0) \|_2 }{ m \lambda }$ & Lemma~\ref{lem:3.3} & Actual distance moved of weight, continuous case  \\ \hline
		$D$ & $\frac{ 4\sqrt{n} \| y - u(0) \|_2 }{ m \lambda }$ & Lemma~\ref{lem:4.1} & Actual distance moved of weight, discrete case  \\ \hline
		$\eta$ & $\lambda/( n^2 m )$ & Eq.~\eqref{eq:choice_of_eta_R} & Step size of gradient descent \\ \hline
		$m_0$ & $\geq \lambda^{-2} n^2 \log(B n/\delta)$ & Lemma~\ref{lem:3.1} & Bounding discrete $H$ and continuous $H$ \\ \hline
		$m_0$ & $\geq R^{-1} \log(B n / \delta)$ & Lemma~\ref{lem:3.2} & Bounding discrete $H(w)$ and discrete $H(w+\Delta w)$ \\ \hline
		$m_0$ & $\geq R^{-1} \log (B n / \delta)$ & Lemma~\ref{cla:C1} & \\ \hline
		$m_0$ & $\geq R^{-1} \log (B n / \delta)$ & Lemma~\ref{cla:C2} & \\ \hline
		$m_0$ & $\geq R^{-1} \log (B n / \delta)$ & Lemma~\ref{cla:C3} & \\ \hline
		$m$ & $\lambda^{-2} n^2 \log^3( m n/\delta)$  & Lemma~\ref{lem:3.4}, Claim~\ref{cla:yu0} & $D < R$ and $\| y - u(0) \|_2^2 = \wt{O}(n)$ \\ \hline
		$m_0$ & $m/2$ &  & The number of different Gaussian vectors \\ \hline
		$B$ & $2$ & & Size of each block \\ \hline
		$T$ & $\lambda^{-2} n^2 \log (1/\epsilon)$ & & \\ \hline
	\end{tabular}
	
\end{table*}

\subsection{Proof of Lemma \ref{lem:4.1}}
\begin{proof}
	We use the norm of the gradient to bound this distance,
	\begin{align*}
	\| w_r(k+1) - w_r(0) \|_2
	\leq & ~ \eta \sum_{i=0}^k \left\| \frac{ \partial L( W(i) ) }{ \partial w_r(i) } \right\|_2 \\
	\leq & ~ \eta \sum_{i=0}^k \Big\| \sum_{j=1}^n ( y_j - u(i)_j ) \cdot a_{r} x_j \cdot {\bf 1}_{ \langle w_r(s), x_j \rangle \geq 0 } \Big\|_2 \\
	\leq & ~ \sum_{i=0}^k \sum_{j=1}^n | y_j - u(i)_j |\\
	\leq & ~ \eta \sum_{i=0}^k   \sqrt{n} \| y - u(i) \|_2  \\
	\leq & ~ \eta \sum_{i=0}^k  \sqrt{n} ( 1 - m \eta \lambda / 2 )^{i/2} \| y - u(0) \|_2 \\
	\leq & ~ \eta \sum_{i=0}^{\infty} \sqrt{n} ( 1 - m \eta \lambda / 2 )^{i/2} \| y - u(0) \|_2  \\
	= & ~  \frac{ 4 \sqrt{n} \| y - u(0) \|_2 }{ m \lambda },
	\end{align*}
	where the first step follows from Eq. \eqref{eq:w_update}, the second step follows from the expression of the gradient (see Eq. \eqref{eq:gradient_bound}), the third step follows from $|a_r|=1$, $\| x_j \|_2=1$ and ${\bf 1}_{\langle w_r(s), x_j \rangle \geq 0} \leq 1$, the fourth step follows from the Cauchy-Schwarz inequality, the fifth step follows from the induction hypothesis, the sixth step relaxes the summation to an infinite summation, and the last step follows from $\sum_{i=0}^{\infty}(1- m \eta\lambda/2)^{i/2}=\frac {2}{m \eta\lambda}$.
	
	Thus, we complete the proof.
\end{proof}
\subsection{Proof of Claim \ref{cla:inductive_claim}}
\begin{proof}
We can rewrite $u(k+1) - u(k) \in \R^n$ in the following sense
\begin{align*}
u(k+1) - u(k) = v_1 + v_2 .
\end{align*}

Then, we can rewrite $v_{1,i} \in \R$ with the notation of $H$ and $H^{\bot}$
\begin{align*}
v_{1,i} 
= & - \eta  \sum_{j=1}^n x_i^\top x_j (u_j - y_j) \sum_{r \in S_i} {\bf 1}_{ w_r(k)^\top x_i \geq 0 , w_r(k)^\top x_j \geq 0 } = - m \eta \sum_{j=1}^n (u_j - y_j) ( H_{i,j}(k) - H_{i,j}^{\bot}(k) ) ,
\end{align*}
which means vector $v_1 \in \R^n$ can be written as
\begin{align}\label{eq:rewrite_v1}
v_1 = m \cdot \eta ( y - u(k) )^\top ( H( k ) - H^{\bot}( k ) ) .
\end{align}

We can rewrite $\| y - u(k+1) \|_2^2$ as follows:
\begin{align*}
 ~\| y - u(k+1) \|_2^2
= & ~ \| y - u(k) - ( u(k+1) - u(k) ) \|_2^2 \\
= & ~ \| y - u(k) \|_2^2 \underbrace{ - 2 ( y - u(k) )^\top  ( u(k+1) - u(k) ) }_{ := C_1 + C_2 + C_3 } 
 + \underbrace{ \| u (k+1) - u(k) \|_2^2 }_{ := C_4 } .
\end{align*}

We can rewrite the second term in the above equation in the following sense,
\begin{align*}
& ~ ( y - u(k) )^\top ( u(k+1) - u(k) ) \\
= & ~ ( y - u(k) )^\top ( v_1 + v_2 )  \\
= & ~ ( y - u(k) )^\top v_1 + ( y - u(k) )^\top v_2  \\
= & ~  \underbrace{ m \eta ( y - u(k) )^\top H(k) ( y - u (k) ) }_{ - C_1 / 2 } ~ \underbrace{ - m \eta ( y - u(k) )^\top H(k)^{\bot} ( y - u(k) ) }_{ - C_2 / 2 } ~ \underbrace{ + ( y - u(k) )^\top v_2 }_{ - C_3 / 2 },
\end{align*}
where the third step follows from Eq.~\eqref{eq:rewrite_v1}.

Thus, we have
\begin{align*}
~\| y - u(k+1) \|_2^2 =& ~ \| y - u(k) \|_2^2 + C_1 + C_2 + C_3 + C_4 \\
\leq & ~ {  \| y - u(k) \|_2^2 ( 1 - m \eta \lambda + 8 m \eta n R  + 8 m \eta n R  + m^2 \eta^2 n^2 ) }
\end{align*}
where the last step follows from Claims~\ref{cla:C1}, \ref{cla:C2}, \ref{cla:C3} and \ref{cla:C4},
whose proofs are given later.
\end{proof}

\subsection{Proof of Claim \ref{cla:yu0}}

 \begin{claim}\label{cla:yu0}
For $0<\delta<1$,
with probability at least $1-\delta$,
\begin{align*}
\|y-u(0)\|_2^2=O(n\log(m/\delta)\log^2(n/\delta)).
\end{align*}
\end{claim}

\begin{proof}

Due to the way we choose $w$ and $a$, it is easy to see that $i(0) = {\bf 0} \in \R^n$. Thus
\begin{align*}
    \| y - u(0) \|_2^2 = \| y \|_2^2 = O(n),
\end{align*}
where the last step follows from $|y_i|=O(1)$ and $y \in \R^n$.

\end{proof}

\subsection{Proof of Claim \ref{cla:C1}}

\begin{claim}\label{cla:C1}
Let $C_1 = -2 m \eta (y - u(k))^\top H(k) ( y - u(k) )$. We have that
\begin{align*}
C_1 \leq - m \eta \lambda\cdot \| y - u(k) \|_2^2 
\end{align*}
holds with probability at least $1-n^2 \cdot B \cdot \exp(-m_0 R /10)$.
\end{claim}

\begin{proof}
By Lemma \ref{lem:3.2} and our choice of $R<\frac{\lambda}{8n}$,
we have $\|H(0)-H(k)\|_F\leq 2n\cdot \frac{\lambda}{8n}=\frac {\lambda}{4}$.
Recall that $\lambda=\lambda_{\min}(H(0))$.
Therefore
\begin{align*}
\lambda_{\min}(H(k)) \geq \lambda_{\min}(H(0))- \|H(0)-H(k)\|\geq \lambda /2.
\end{align*}
Then we have
\begin{align*}
  (y - u(k))^\top H(k) ( y - u(k) ) \geq \| y - u(k) \|_2^2 \cdot \lambda / 2.
\end{align*}
Thus, we complete the proof.
\end{proof}

\subsection{Proof of Claim \ref{cla:C2}}

\begin{claim}\label{cla:C2}
Let $C_2 = 2 m \cdot \eta ( y - u(k) )^\top H(k)^{\bot} ( y - u(k) )$. We have that
\begin{align*}
C_2 \leq 8m \cdot \eta  n R \cdot \| y - u(k) \|_2^2
\end{align*}
holds with probability $1 - n \cdot B \cdot \exp(-m_0 R)$.
\end{claim}

\begin{proof}
Note that
\begin{align*}
C_2 \leq 2 \eta \| y - u(k) \|_2^2 \| H(k)^{\bot} \|.
\end{align*}

We thus need an upper bound on $\| H(k)^{\bot} \|$. Since $\| \cdot \| \leq \| \cdot \|_F$, it suffices to upper bound $\| \cdot \|_F$.

For each $i \in [n]$, we define $y_i$ as follows
\begin{align*}
y_i=\sum_{r=1}^m\mathbf{1}_{r\in \ov{S}_i} .
\end{align*}
For each $i\in [n]$, $b\in [B]$, we define
\begin{align*}
y_i^b=\sum_{r=1}^{m_0}\mathbf{1}_{r\in \ov{S}_i} .
\end{align*}

Using Fact~\ref{fact:bound_H_k_bot}, we have $\| H(k)^{\bot} \|_2 \leq \frac{n}{m^2} \sum_{i=1}^n y_i^2 $.

Fix $i \in [n]$. Our plan is to use Bernstein's inequality (Lemma \ref{lem:bernstein}) to upper bound $y_i$ with high probability.

First by Eq.~\eqref{eq:Air_bound} we have 
$
\E[\mathbf{1}_{r\in \ov{S}_i}]\leq R 
$. 
We also have
\begin{align*}
\E \left[(\mathbf{1}_{r\in \ov{S}_i}-\E[\mathbf{1}_{r\in \ov{S}_i}])^2 \right]
 = & ~ \E[\mathbf{1}_{r\in \ov{S}_i}^2]-\E[\mathbf{1}_{r\in \ov{S}_i}]^2\\
\leq & ~ \E[\mathbf{1}_{r\in \ov{S}_i}^2] \\
\leq & ~ R .
\end{align*}
Finally we have $|\mathbf{1}_{r\in \ov{S}_i}-\E[\mathbf{1}_{r\in \ov{S}_i}]|\leq 1$.

Notice that $\{\mathbf{1}_{r\in \ov{S}_i}\}_{r=1}^{m_0}$ are mutually independent,
since $\mathbf{1}_{r\in \ov{S}_i}$ only depends on $w_r(0)$.
Hence from Bernstein's inequality (Lemma \ref{lem:bernstein}) we have for all $t>0$,
\begin{align*}
\Pr \left[ y_i > m_0\cdot R+t \right] \leq \exp \left(-\frac{ t^2/2 }{ m_0 \cdot R + t/3} \right).
\end{align*}
By setting $t=3 m_0 R$, we have
\begin{align}\label{eq:Si_size_bound}
\Pr \left[ y_i^b > 4 m_0 R \right] \leq \exp( - m_0 R ).
\end{align}
Since we have $B$ such copies of the above inequality, it follows that
\begin{align*}
\Pr \left[ y_i > 4 m R \right]
= & ~ \Pr \left[ \sum_{b=1}^B y_i^b > 4 m_0 R \cdot B\right] 
\leq  ~ B \cdot \exp( - m_0 R )
\end{align*}

Hence by a union bound, with probability at least $1-n \cdot B \cdot \exp(-m_0R)$,
\begin{align*}
\| H(k)^{\bot} \|_F^2\leq \frac{n}{m^2}\cdot n\cdot (4mR)^2=16n^2R^2 .
\end{align*}
Putting it all together we have
\begin{align*}
\| H(k)^{\bot} \|\leq \| H(k)^{\bot} \|_F\leq 4nR
\end{align*}
with probability at least $1-n \cdot B \cdot \exp( - m_0 R )$.

\end{proof}

\subsection{Proof of Claim \ref{cla:C3}}\label{sec:proof_c3}

\begin{claim}\label{cla:C3}
Let $C_3 = - 2 (y - u(k))^\top v_2$. Then we have
\begin{align*}
C_3 \leq 8 m \eta n R \cdot \| y - u(k) \|_2^2  
\end{align*}
with probability at least $1 - n \cdot B \cdot \exp( - m_0 R )$.
\end{claim}

\begin{proof}
Using the Cauchy-Schwarz inequality, we have
$
C_3 \leq 2 \| y - u(k) \|_2 \cdot \| v_2 \|_2
$. 
We can upper bound $\| v_2 \|_2$ in the following way
\begin{align*}
\| v_2 \|_2^2
\leq &~ \sum_{i=1}^n \left( \eta  \sum_{ r \in \ov{S}_i } \left| \left( \frac{ \partial L(W(k)) }{ \partial w_r(k) } \right)^\top x_i \right|\right)^2\\
= &~  \eta^2 \sum_{i=1}^n \left(\sum_{r=1}^m \mathbf{1}_{r\in \ov{S}_i}\left| \left( \frac{ \partial L(W(k)) }{ \partial w_r(k) } \right)^\top x_i \right|\right)^2\\
\leq &~ \eta^2 \cdot \max_{r \in [m]} \left|  \frac{ \partial L(W(k)) }{ \partial w_r(k) } \right|^2\cdot\sum_{i=1}^n \left(\sum_{r=1}^m \mathbf{1}_{r\in \ov{S}_i}\right)^2\\
 \leq & ~  \eta^2 \cdot (  \sqrt{n}  \| u(k) - y\|_2 )^2 \cdot \sum_{i=1}^n \left(\sum_{r=1}^m \mathbf{1}_{r\in \ov{S}_i}\right)^2\\
  \leq & ~  \eta^2 \cdot ( \sqrt{n}  \| u(k) - y\|_2 )^2 \cdot \sum_{i=1}^n (4mR)^2
  =  ~ 16m^2 n^2R^2\eta^2\| u(k) - y\|_2^2,
\end{align*}
where the first step follows from the definition of $v_2$, the fourth step follows from $\max_{r \in [m]} | \frac{ \partial L (W(k)) }{ \partial w_r(k) } | \leq \sqrt{n} \cdot \| u(k) - y \|_2$, and the fifth step follows from $\sum_{r=1}^m {\bf 1}_{r \in \ov{S}_i } \leq 4 m R$ which holds with probability at least $1 - n \cdot B \cdot \exp( - m_0 R )$.
\end{proof}

\subsection{Proof of Claim \ref{cla:C4}}

\begin{claim}\label{cla:C4}
Let $C_4  = \| u (k+1) - u(k) \|_2^2$. Then we have
\begin{align*}
C_4 \leq m^2 \cdot \eta^2 n^2 \cdot \| y - u(k) \|_2^2.
\end{align*}
\end{claim}

\begin{proof}
We have
\begin{align*}
C_4 \leq & ~ \eta^2 \sum_{i=1}^n \left( \sum_{r=1}^m \Big\| \frac{ \partial L( W(k) ) }{ \partial w_r(k) } \Big\|_2 \right)^2 \leq  ~ m^2 \cdot \eta^2 n^2 \| y - u(k) \|_2^2.
\end{align*}
where the first step follows from Eq.~\eqref{eq:w_update}
and the last step follows from Eq.~\eqref{eq:gradient_bound}.
\end{proof}

\subsection{Proof of Fact \ref{fact:bound_H_k_bot}}

\begin{fact}\label{fact:bound_H_k_bot}
Let $H(k)^{\bot}$ be defined as in Eq.~\eqref{eq:def_H(k)^bot}. Then we have
	\begin{align*}
	\| H(k)^{\bot} \|_2 \leq \frac{n}{m^2} \sum_{i=1}^n y_i^2 .
	\end{align*}
\end{fact}

\begin{proof}
	We have
	\begin{align*}
	\| H(k)^{\bot} \|_F^2
	= & ~ \sum_{i=1}^n\sum_{j=1}^n (H(k)^{\bot}_{i,j})^2\\
	= & ~ \sum_{i=1}^n\sum_{j=1}^n \Big( \frac {1} {m}\sum_{r\in \ov{S}_i} x_i^\top x_j\mathbf{1}_{w_r(k)^\top x_i\geq 0,w_r(k)^\top x_j\geq 0} \Big)^2\\
	= & ~ \sum_{i=1}^n\sum_{j=1}^n \Big( \frac {1} {m}\sum_{r=1}^m x_i^\top x_j\mathbf{1}_{w_r(k)^\top x_i\geq 0,w_r(k)^\top x_j\geq 0} \cdot \mathbf{1}_{r\in \ov{S}_i} \Big)^2\\
	= & ~ \sum_{i=1}^n\sum_{j=1}^n \left( \frac {x_i^\top x_j} {m} \right)^2 \Big( \sum_{r=1}^m \mathbf{1}_{w_r(k)^\top x_i\geq 0,w_r(k)^\top x_j\geq 0} \cdot \mathbf{1}_{r\in \ov{S}_i} \Big)^2 \\
	\leq & ~ \frac{1}{m^2} \sum_{i=1}^n\sum_{j=1}^n \Big( \sum_{r=1}^m \mathbf{1}_{w_r(k)^\top x_i\geq 0,w_r(k)^\top x_j\geq 0} \cdot \mathbf{1}_{r\in \ov{S}_i} \Big)^2 \\
	= & ~ \frac{n}{m^2} \sum_{i=1}^n \Big( \sum_{r=1}^m \mathbf{1}_{r\in \ov{S}_i} \Big)^2 \\
	= & ~ \frac{n}{m^2} \sum_{i=1}^n y_i^2 .
	\end{align*}
	where the only inequality follows from $\| x_i \|_2 , \| x_j \|_2 \leq 1$.
\end{proof}